\documentclass[11pt]{article}

\usepackage{microtype}
\usepackage{graphicx}
\usepackage{subcaption}
\usepackage{booktabs}
\usepackage{microsoft-tech-report}
\usepackage{hyperref}
\usepackage{url}
\usepackage[utf8]{inputenc}
\usepackage[T1]{fontenc}
\usepackage{amsmath}
\usepackage{amssymb}
\usepackage{amsfonts}
\usepackage{mathtools}
\usepackage{amsthm}
\usepackage{bm}
\usepackage{nicefrac}
\usepackage{enumitem}
\usepackage{multirow}
\usepackage[capitalize,noabbrev]{cleveref}
\usepackage{xspace}
\usepackage{pifont}
\usepackage{wrapfig}
\usepackage{algorithm}
\usepackage{algpseudocode}
\usepackage{amsthm}
\newtheorem{theorem}{Theorem}
\newtheorem{lemma}{Lemma}
\newtheorem{proposition}{Proposition}
\newtheorem{corollary}{Corollary}

\usepackage{placeins}
\usepackage{booktabs}
\usepackage{multirow}
\usepackage{graphicx}
\graphicspath{{figs/}}
\usepackage{adjustbox}
\usepackage[table]{xcolor}

\newcommand{\best}[1]{\textbf{#1}}
\definecolor{oursgray}{RGB}{242,242,242}

\definecolor{gainred}{RGB}{180,60,60}

\newcommand{\gain}[1]{%
  {\scriptsize\textcolor{gainred}{\hspace{1pt}+#1}}%
}

\newcommand{\methodspace}{\addlinespace[1.2pt]}

\definecolor{deltapos}{HTML}{0E8A4A}
\definecolor{deltaneg}{HTML}{B91C1C}

\techreportlabel{Microsoft Research}
\techreportshorttitle{VidForensics-M1}

\hypersetup{
  colorlinks=true,
  linkcolor=msftblue,
  citecolor=msftblue,
  urlcolor=msftblue,
  pdftitle={VidForensics-M1: Meta-Detection Reinforcement Learning with Verifiable Temporal Grounding for AI-Generated Video Forensics}
}

\begin{document}

\thispagestyle{empty}


\noindent
\begin{minipage}[c]{0.55\linewidth}
\raggedright

\includegraphics[height=1.5cm]{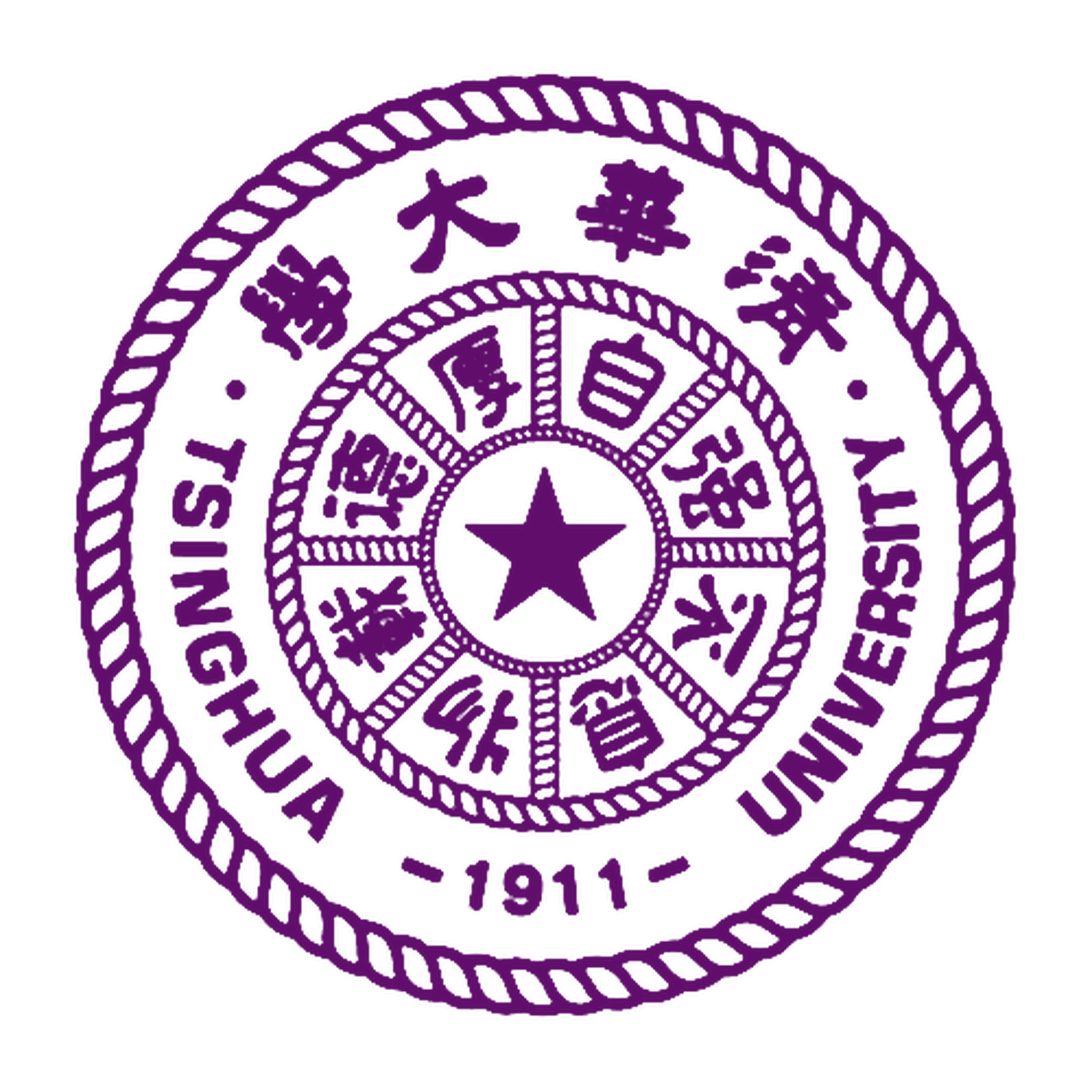}
\hspace{0.25cm}
\includegraphics[height=1.5cm]{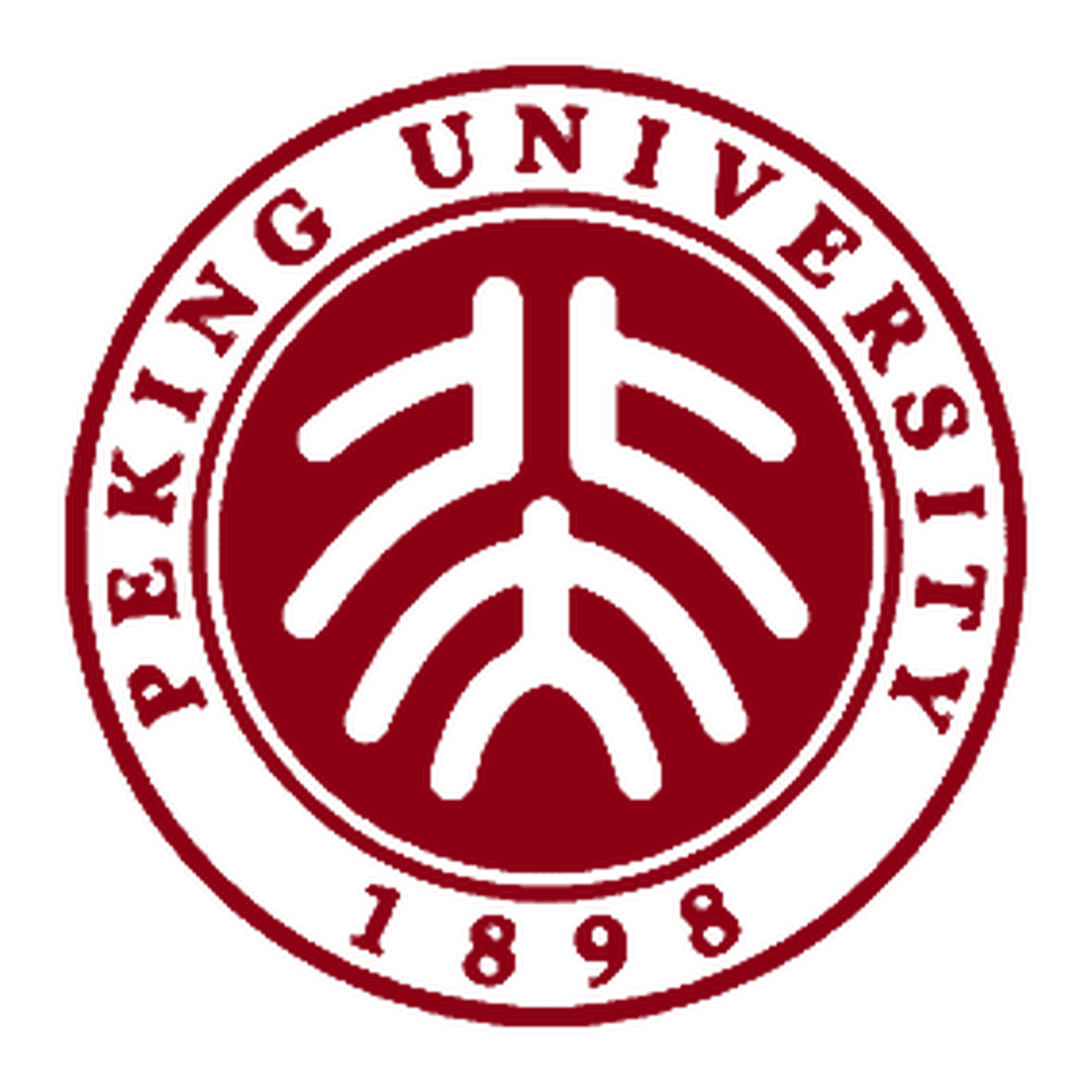}
\hspace{0.25cm}
\includegraphics[height=1.5cm]{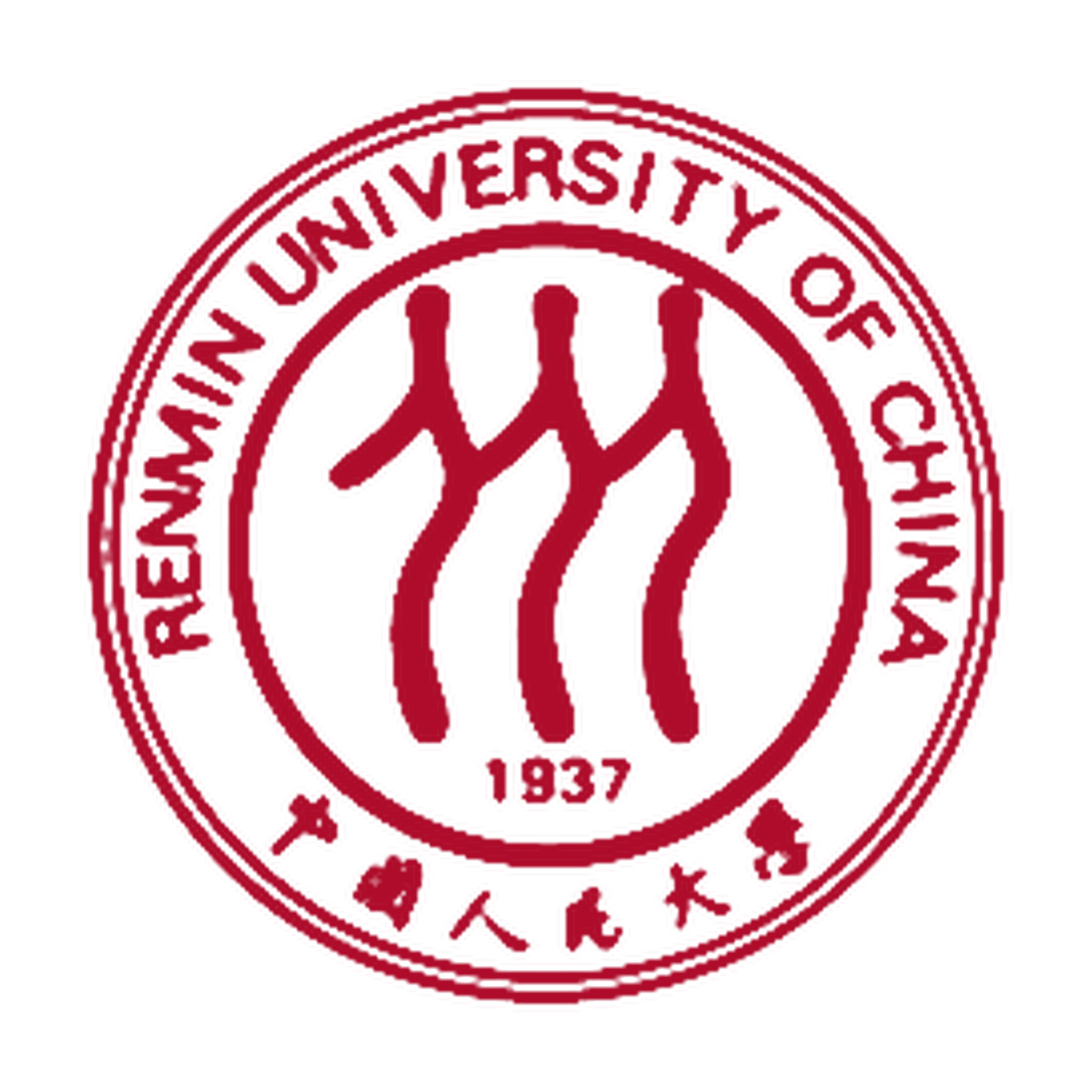}

\end{minipage}%
\begin{minipage}[c]{0.44\linewidth}
\raggedleft

\includegraphics[height=0.65cm]{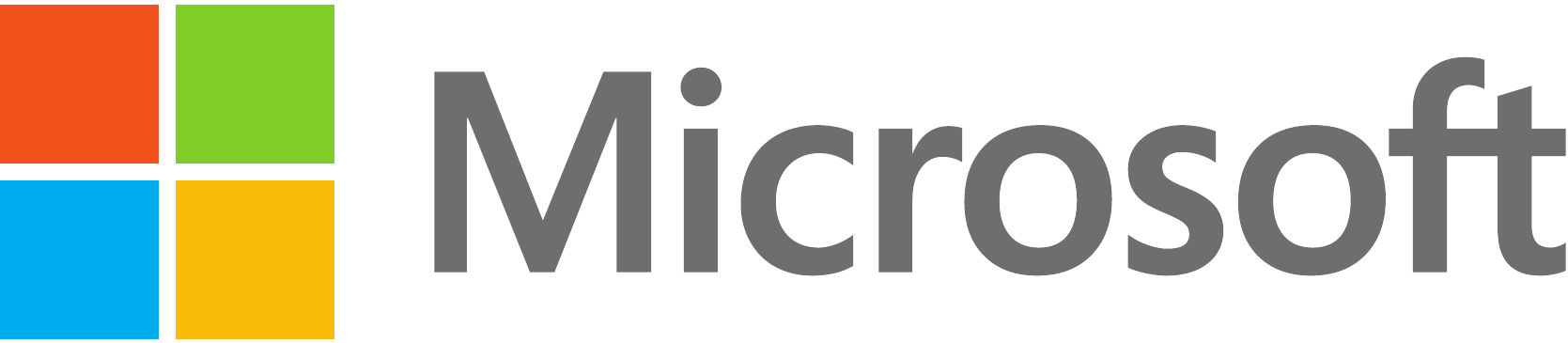}

\vspace{0.15cm}


\end{minipage}\par

\vspace{0.35em}

\noindent{\color{msftline}\rule{\linewidth}{0.8pt}\par}

\vspace{1.0em}

\begin{center}

{{\msfttitlefont\fontsize{13}{20}\selectfont\color{msftdark}
VidForensics-M1: Meta-Detection Reinforcement Learning with Verifiable Temporal Grounding for AI-Generated Video Forensics\par}}

\vspace{1.25em}

{\normalsize\rmfamily\color{msftdark}

Bowei Liu$^{1,*}$ \hspace{0.6em}
Zheng Lu$^{2,*}$ \hspace{0.6em}
Yuhan Bian$^{3,*}$ \hspace{0.6em}
Xinchen Zhang$^{1,*}$ \hspace{0.6em}
Xingming Shui$^{1}$\\[-0.1em]

Yuesheng Huang$^{1}$ \hspace{0.6em}
Xuhuan Li$^{1}$ \hspace{0.6em}
Zihao Liu$^{1}$ \hspace{0.6em}
Yifan Yang$^{4}$ \hspace{0.6em}
Jun Zhou$^{1}$ \hspace{0.6em}
Xiu Li$^{1,\ddagger}$\par

}

\vspace{0.22cm}

{\footnotesize\rmfamily\color{msftgray}

$^{1}$ Tsinghua University \quad
$^{2}$ Peking University \quad
$^{3}$ Renmin University of China \quad
$^{4}$ Microsoft\par

}

{\footnotesize\rmfamily\itshape\color{msftgray}

$^{*}$: Equal Contribution. \ \  
$^{\ddagger}$: Corresponding Authors.\par

}

\end{center}

\vspace{0.45em}

\begin{msfttitlebox}

\setlength{\parindent}{0cm}
\setlength{\parskip}{0.14cm}
\raggedright
\nohyphens

Recent advances in video generation models have dramatically improved the realism of synthetic videos, blurring the boundary between generated and authentic content and raising significant concerns about misinformation.
Existing MLLM-based detectors predominantly rely on supervised fine-tuning or label-level reinforcement learning, where predefined or coarse-grained supervision signals limit their generalization to out-of-domain scenarios and emerging video generators. 
To overcome these limitations, we are the first to introduce the concept of \textbf{meta-detection} into AI-generated video detection, which enables reliable forgery detection by jointly evaluating the predicted label and the supporting evidence within reinforcement learning.
This paradigm presents two core challenges: 
(1) identifying evidence forms that provide reliable and scalable supervision for meta-detection, and 
(2) developing effective mechanisms to integrate such evidence into label-level reinforcement learning for trustworthy and generalizable synthetic video detection.
Textual rationales offer semantically rich descriptions of forgery artifacts, yet their generation and verification rely heavily on external reference models, making the resulting supervision vulnerable to hallucinations and semantic biases. 
In contrast, temporal grounding provides a more objective and verifiable signal, as manipulated temporal intervals can be precisely determined through controlled forgery construction. 
Based on this, we propose an automated data construction pipeline that generates paired real-fake videos by reconstructing and replacing temporal segments using boundary-frame-conditioned video generation models.
Furthermore, we then propose \textbf{Evidence-Guided Reward Redistribution}, which performs evidence-aware credit assignment by redistributing rewards among label-correct responses according to their evidence quality, thereby preserving reliable label supervision while progressively encouraging the detector to acquire fine-grained and verifiable forgery localization capabilities.
Extensive experiments demonstrate that \textbf{VidForensics-M1} effectively leverages verifiable temporal evidence to achieve more robust and generalizable AI-generated video detection.

\vspace{0.14cm}

{\setlength{\parskip}{0.06cm}\small



}

\vspace{-0.08cm}

\end{msfttitlebox}


\section{Introduction}

The rapid evolution of generative foundation models has significantly advanced the visual fidelity of AI-generated videos \cite{hou2026survey, li2026cubecomposer,xiong2026evatok, xu2025smrabooth,xu2025hunyuanportrait,zhang2026zo3t}, making them increasingly difficult to distinguish from authentic content. While these models have enabled unprecedented progress in digital content creation, their widespread accessibility also introduces significant challenges to media authenticity and misinformation mitigation \cite{davodi2026perceptual}. Consequently, developing reliable methods for AI-generated video authenticity verification has emerged as a critical task for ensuring trustworthy digital media \cite{vaccari2020deepfakes, chandra2024reducing}.

\begin{wrapfigure}{r}{0.5\textwidth}
\vspace{-4mm}
    \centering
    \includegraphics[width=0.5\textwidth]{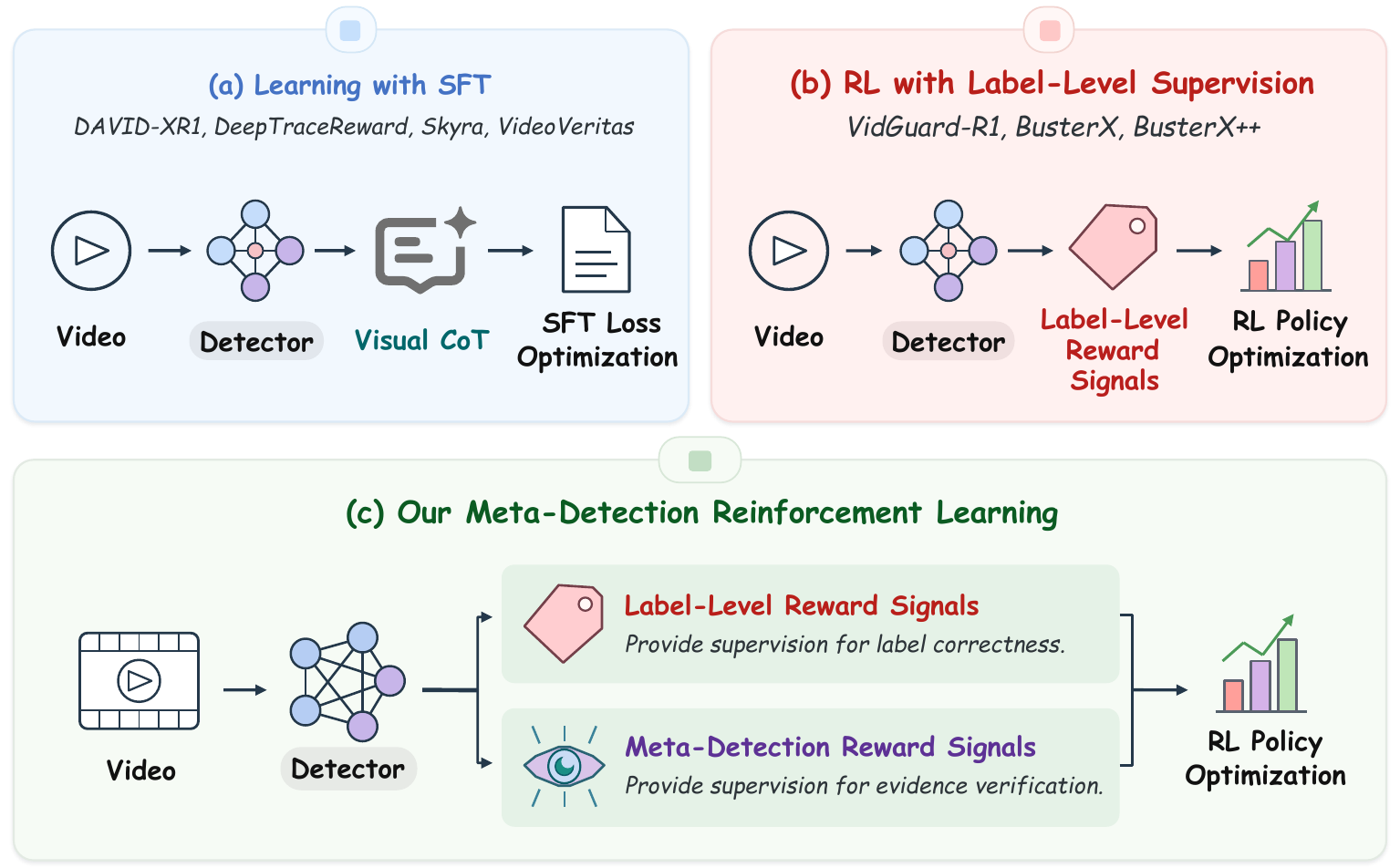}
    \vspace{-5mm}
    \caption{Comparison of MLLM-Based AI-Generated Video Detection Methods.}
    \label{fig1}
    \vspace{-3mm}
\end{wrapfigure}

As shown in Fig. \ref{fig1}, existing MLLM-based detection methods can be broadly categorized into two paradigms. The first paradigm learns fine-grained artifact perception through supervised fine-tuning on visual chain-of-thought annotations provided by human annotators or teacher models \cite{gao2025davidxr1, fu2025deeptracereward, li2026skyra,tan2026videoveritas, corvi2025seeing, NEURIPS2022_9d560961}. Although these approaches enable detectors to capture fine-grained visual forgery artifacts, imitating predefined reasoning trajectories can lead to overfitting to specific annotation patterns or training distributions, thereby limiting their generalization ability in open-world scenarios. The second paradigm leverages reinforcement learning to improve generalization across diverse video generation models and data distributions. However, existing approaches mainly rely on label-level supervision signals, which may encourage detectors to exploit spurious correlations or superficial cues rather than identifying the underlying forgery artifacts \cite{park2026vidguardr, wen2025busterxpp, 3294996.3295184, amodei2016concreteproblemsaisafety, Guo_2025}.

Motivated by these limitations, we introduce evidence-aware supervision into reinforcement learning and term this paradigm \textbf{meta-detection}. By jointly optimizing label correctness and evidence validity, meta-detection enables more reliable synthetic video detection. This evidence-aware feedback provides a more informative and verifiable learning signal for policy optimization, guiding detectors toward identifying and localizing visual forgery artifacts and ultimately achieving more robust and generalizable AI-generated video detection.

A core question is what forms of evidence can serve as trustworthy and scalable supervision for meta-detection? Although textual rationales provide semantically rich descriptions of forgery artifacts, they are typically generated by powerful reference models, meaning that the resulting supervision may inherit the biases and limitations of the teacher models rather than represent objective and verifiable ground truth \cite{gao2025davidxr1, fu2025deeptracereward, turpin2023language, li2026skyra, he2021forgerynet}. In contrast, temporal grounding provides a more objective and verifiable supervision signal, as the manipulated temporal intervals are explicitly determined through the controlled forgery construction process \cite{wang-etal-2025-grounded}.

Based on this, we propose an automated and scalable data construction pipeline that generates paired real-fake videos with verifiable temporal evidence for meta-detection. As shown in Fig. \ref{fig2}, we randomly sample a temporal interval from each real video and extract its preceding and succeeding segments as boundary contexts. The corresponding boundary frames are then provided to video generation models to reconstruct the missing content, which is inserted between the two contexts to form the corresponding fake video. This controlled construction process naturally provides ground-truth manipulated temporal intervals for each forged video. Furthermore, we leverage strong  models to describe observable forgery artifacts in each fake video, including geometric deformation, temporal inconsistency, and physical violation.

Another core question is how to effectively incorporate meta-detection feedback into label-level reinforcement learning to enable trustworthy and generalizable synthetic video detection?  To address this, we propose Evidence-Guided Reward Redistribution (EGRR), which reallocates rewards among label-correct responses according to their evidence quality. As shown in Fig. \ref{fig3}, EGRR maintains stable label-level optimization while calibrating the reward distribution based on evidence reliability, encouraging detectors to recognize fine-grained visual forgery artifacts.

Our contributions can be summarized as follows:
\begin{itemize}
    \item We are the first to introduce the concept of \textbf{meta-detection} into AI-generated video detection and incorporate it into reinforcement learning as an evidence-aware feedback signal.

    \item We demonstrate that rule-based temporal grounding provides more reliable and verifiable feedback than model-based textual rationales, making it a more suitable evidence source for meta-detection.

    \item We propose an automated and scalable data construction pipeline through boundary-frame-conditioned temporal segment reconstruction and replacement, which constructs paired real-fake videos while providing ground-truth manipulated temporal intervals.

    \item We propose Evidence-Guided Reward Redistribution (EGRR), which reallocates rewards among label-correct responses according to their evidence quality, enabling effective calibration of label-level reinforcement learning with meta-detection feedback.
\end{itemize}

\section{Related Work}

\paragraph{AI-Generated Video Detection Methods.}
Early research on visual AIGC forensics predominantly explored
image-level forgery detection \cite{wang2020cnn,ojha2023universal, yan2024transcending, tan2024rethinking, nguyen2024laa, fu2025exploring,yan2025sanity,yang2025all}.
As video generation models continue to advance, increasing attention has
been devoted to detecting AI-generated videos \cite{hou2026survey}.
DAVID-XR1 \cite{gao2025davidxr1} distills teacher-generated visual chains of thought grounded in
defect categories and spatio-temporal annotations.
DeeptraceReward \cite{fu2025deeptracereward} trains multimodal reward models from human rationales,
bounding boxes, and temporal intervals.
Skyra \cite{li2026skyra} generates grounded artifact analyses containing temporal intervals
and spatial locations.
VidGuard-R1 \cite{park2026vidguardr} combines an SFT cold start with GRPO and specialized rewards
for temporal artifacts and generation complexity.
BusterX \cite{wen2025busterx} bypasses SFT and trains directly with reinforcement learning using
video clips and binary authenticity labels.
BusterX++ \cite{wen2025busterxpp} extends this pure-RL formulation to unified image and video
detection, preserving policy entropy to encourage cross-modal
exploration.
VideoVeritas \cite{tan2026videoveritas} improves fine-grained video perception through automatically
verifiable pretext tasks, including grounding and counting, rather than
directly optimizing the detection objective.
Unlike existing approaches, we introduce \textbf{meta-detection} into AI-generated video detection, where evidence-aware feedback is incorporated into reinforcement learning to jointly optimize prediction correctness and evidence validity.

\paragraph{AI-Generated Video Detection Datasets.}
With the rapid advancement of generative models \cite{zhang2026omniverifier, zhang2026generative, zhang2024realcompo,zhang2025itercomp}, an increasing number of datasets have been introduced to facilitate research on AI-generated video detection \cite{fu2025learning, bai2024gvd, chen2026genvideo, ma2025gvf, wen2025busterx, gao2025davidxr1,fu2025deeptracereward,li2026skyra}.
GVD \cite{bai2024gvd} collects videos generated by diverse text-to-video and image-to-video
models, while GenVideo \cite{chen2026genvideo} scales the benchmark to millions
of real and synthetic videos.
GVF \cite{ma2025gvf} mitigates semantic shortcuts by constructing content-matched fake
videos from prompts extracted from the corresponding real
videos.
GenVidBench \cite{ni2026genvidbench} introduces cross-source and cross-generator evaluation
settings together with semantic annotations.
GenBuster-200K \cite{wen2025busterx} incorporates recent video generation models and
large-scale real-world videos for training MLLM-based
detectors.
Recent evidence-oriented datasets, including DAVID-X \cite{gao2025davidxr1}, DeeptraceReward \cite{fu2025deeptracereward},
and ViF-CoT-4K \cite{li2026skyra}, provide artifact rationales and spatio-temporal
annotations,
but do not derive exact manipulated intervals from a controlled
segment-replacement process. We address this gap by constructing paired
real-fake videos through controlled segment manipulation with ground-truth
temporal localization.
\section{From Model-Based Rationales to Rule-Based Temporal Feedback}

\begin{figure*}[t]
    \centering
    \includegraphics[width=\linewidth]{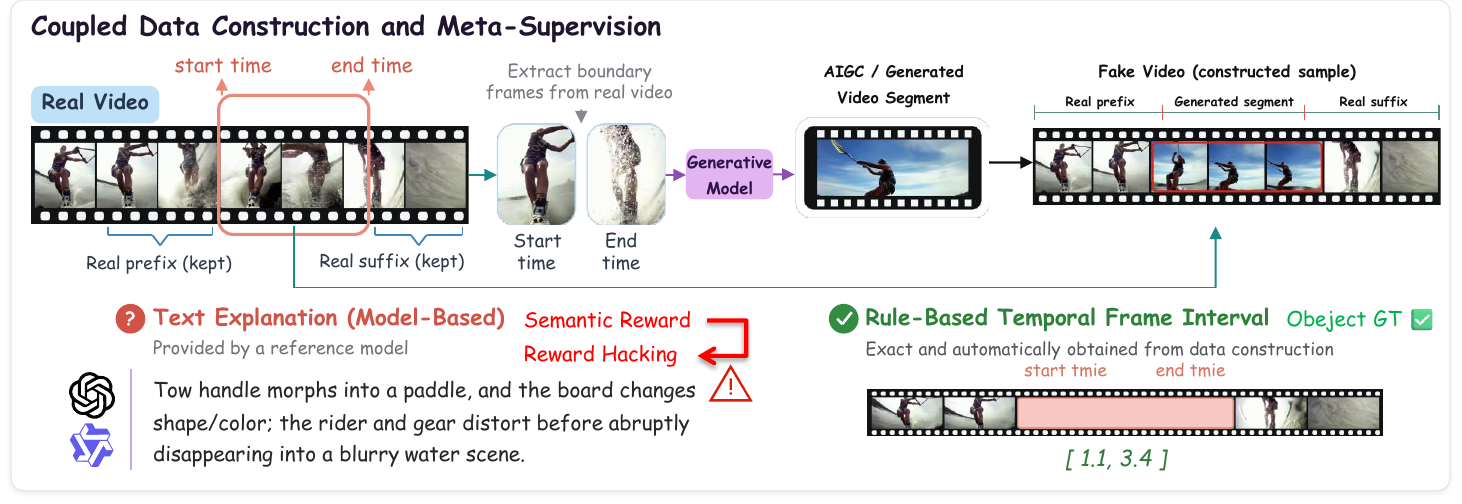}
    \vspace{-3mm}
    \caption{Overview of our automated and scalable data construction pipeline and comparison between two meta-detection feedback signals.}
    \label{fig2}
    \vspace{-1mm}
\end{figure*}

In this section, we first motivate the introduction of meta-detection into reinforcement learning. We then analyze the limitations of textual explanations as meta-detection feedback signals and introduce rule-based temporal grounding as a more reliable evidence source. Finally, we present an automated and scalable data construction pipeline for generating paired real-fake videos with verifiable temporal forgery localization annotations.

\subsection{Problem Definition}
Traditional RLVR-based video detectors are trained on a dataset
$\mathcal{D}=\{(V_n,y_n)\}_{n=1}^{N}$, where each video
$V_n=\{f_{n,t}\}_{t=1}^{T}$ consists of $T$ frames and
$y_n\in\{\mathrm{real},\mathrm{fake}\}$ denotes its authenticity label.
For each training video $V_i$, Group Relative Policy Optimization samples $G$ responses
$\{o_{ij}\}_{j=1}^{G}$ from the current policy:
$o_{ij}\sim\pi_\theta(\cdot\mid V_i)$.
The predicted label $\hat{y}_{ij}$ is parsed from response $o_{ij}$.
Its label correctness is defined as
$\mathbb{I}[\hat{y}_{ij}=y_i]\in\{0,1\}$.
Traditional label-level RLVR uses the following reward:
\begin{equation}
R^{(i,j)}
=
\lambda_{\mathrm{format}} R_{\mathrm{format}}^{(i,j)}
+
\left(1-\lambda_{\mathrm{format}}\right)
R_{\mathrm{label}}^{(i,j)} .
\label{eq:traditional_rlvr_reward}
\end{equation}
Although this reward encourages models to follow the required output format and predict correct labels, it cannot distinguish whether detectors genuinely identify visual forgery artifacts or obtain rewards by exploiting unreliable shortcuts.
To address this limitation, we introduce the concept of \textbf{meta-detection}, which extracts the predicted label $\hat{y}_{ij}$ and decision evidence $\hat{e}_{ij}$ from response $o_{ij}$, aiming to jointly supervise binary label prediction and evidence trustworthiness for more reliable reinforcement learning.

\subsection{Drawbacks of Model-Based Textual Rationales}
Textual explanations offer semantically rich descriptions of visual forgery artifacts, including geometric deformation, temporal inconsistency, and physical violation. However, these explanations are typically generated by powerful reference models, making the resulting supervision reflect model-dependent interpretations of forgery traces rather than objective ground truth and potentially inheriting the biases and limitations of the reference models. Furthermore, evaluating generated textual explanations typically requires an additional auxiliary model to measure their semantic consistency with reference explanations produced by the reference model. This verification process not only introduces additional sources of model-induced bias but also incurs substantial computational overhead. Consequently, model-based textual explanations remain vulnerable to subjective biases during both generation and evaluation, potentially leading to reward hacking.

\begin{figure*}[t]
    \centering
    \includegraphics[width=\linewidth]{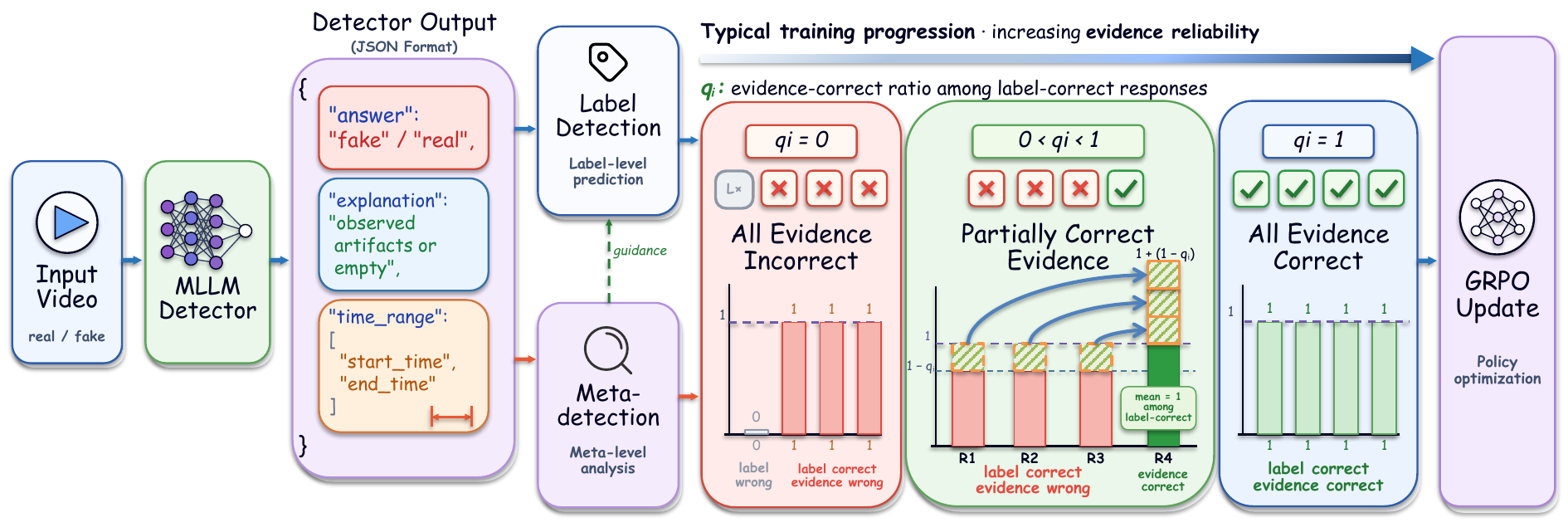}
    \vspace{-3mm}
    \caption{Overview of the Evidence-Guided Reward Redistribution (EGRR) pipeline. The detector takes videos as input and generates JSON-formatted responses containing both label-level and meta-detection signals. The meta-detection signal redistributes rewards among label-correct responses according to evidence quality, which are then used for GRPO policy optimization. The training process evolves from incorrect labels with invalid evidence, to correct labels with mixed evidence quality, and finally to jointly correct labels and reliable evidence.}
    \label{fig3}
    \vspace{-1mm}
\end{figure*}

\subsection{Trustworthy Rule-Based Temporal Grounding and Automated Data Construction Pipeline}
In contrast, temporal grounding provides a more objective and verifiable supervision signal, as the manipulated temporal intervals can be explicitly determined through a controlled forgery construction process. Based on this observation, we propose an automated and scalable data construction pipeline that automatically obtains ground-truth forgery intervals during the construction process.
As shown in Fig. \ref{fig2}, we collect real videos from InternVid and ActivityNet and uniformly process them into 5-second clips. For each real video, we randomly remove a temporal segment and reconstruct the missing content using boundary frames and video generation models. The generated segments are then aligned with the removed segments in terms of duration, resolution, and frame rate before being inserted back into the original temporal locations, producing the corresponding fake videos.
To compare temporal grounding and textual explanations as evidence sources for meta-detection, we employ GPT-5.5 to generate textual descriptions of observable forgery artifacts for each fake video, including geometric deformation, temporal inconsistency, and physical violation. We further use Gemini-3.1-Pro to filter out incorrectly generated explanations.
The quality of temporal evidence is quantified by computing the Intersection-over-Union (IoU) between the predicted temporal interval and the reference interval. Specifically, the predicted and reference intervals are represented as $\hat{t}_{ij}=(\hat{s}_{ij},\hat{e}_{ij})$ and $t_i^{*}=(s_i^{*},e_i^{*})$, respectively, where $s$ and $e$ denote the start and end timestamps.
\begin{equation}
Q_{\mathrm{tem}}
=
\mathrm{IoU}(\hat{t}_{ij},t_i^{*})
=
\frac{
\left|
[\hat{s}_{ij},\hat{e}_{ij}]
\cap
[s_i^{*},e_i^{*}]
\right|
}{
\left|
[\hat{s}_{ij},\hat{e}_{ij}]
\cup
[s_i^{*},e_i^{*}]
\right|
}.
\label{eq:temporal_iou}
\end{equation}
To enable a fair comparison with the binary semantic consistency scores used for textual explanations, we further adopt an IoU threshold of 0.7 to determine whether the temporal evidence is considered correct.

Our automated data construction pipeline provides the following advantages:

\begin{enumerate}

\item Our pipeline provides two types of meta-detection evidence, including manipulated temporal intervals and semantic descriptions, while supporting fully automated and scalable data generation.

\item Temporal evidence is automatically derived from the rule-based construction process, providing reliable and auditable feedback signals that substantially reduce the risk of reward hacking.

\item Unlike fake videos independently generated from real videos, our paired real-fake videos share identical source content and temporal contexts, minimizing semantic shortcuts caused by content-level discrepancies.

\item Unlike approaches that generate entire fake videos, our method introduces localized and controllable manipulations within authentic videos, making the constructed forgeries better reflect complex real-world video forgery scenarios.

\end{enumerate}

\section{Evidence-guided Reward Redistribution}
In this section, we first motivate the introduction of Evidence-Guided Reward Redistribution (EGRR). We then present the detailed algorithm of EGRR and finally provide theoretical analysis demonstrating that EGRR preserves the original label-level learning objective.

\subsection{Beyond Label-Level Feedback}
Label-level RLVR provides feedback to different responses based solely on prediction correctness. Although such binary rewards are accurate, they are overly coarse-grained and provide insufficient guidance to distinguish whether detectors genuinely identify visual forgery artifacts or exploit superficial shortcuts. This limitation makes models vulnerable to reward hacking and restricts their generalization ability on out-of-distribution data.

Meta-detection addresses this limitation by introducing evidence-aware feedback signals that assess the reliability of label-level rewards. Specifically, evidence quality provides an additional dimension to distinguish responses with identical label correctness but different levels of supporting evidence. This naturally motivates us to refine the rewards of label-correct responses according to their evidence quality, transferring reward credits from responses with unreliable evidence to those supported by stronger evidence.

To preserve the original label-level optimization objective, we impose two principles during reward redistribution: (1) responses with incorrect labels remain unchanged, as they should not receive additional evidence-based credit; and (2) the total reward among label-correct responses is preserved. Under these constraints, evidence quality only determines the relative allocation of rewards within the label-correct subset, introducing evidence-based preference while maintaining the original label supervision.

\subsection{Group-Relative Reward Calibration}
Based on the above analysis, we propose Evidence-Guided Reward Redistribution (EGRR), which reallocates rewards among label-correct responses according to their evidence quality while preserving the average reward of correctly classified samples. The objective of EGRR is to refine the original label-level learning signal with evidence-aware preferences without altering the underlying label optimization objective.

Since forgery evidence is only meaningful for fake samples, we focus reward redistribution on responses sampled from videos with $y_i=\mathrm{fake}$. For each response $o_{ij}$, we define the label correctness indicator as $\ell_{ij}=\mathbb{I}[\hat{y}_{ij}=y_i]\in\{0,1\}$, and define the evidence correctness indicator as $e_{ij}\in\{0,1\}$. Specifically, $e_{ij}=1$ indicates that the generated evidence is sufficiently reliable and provides valid support for the predicted label.

We define $q_i$ as the average evidence quality among label-correct responses within the sampled group for the $i$-th fake video:
\begin{equation}
q_i =
\begin{cases}
\displaystyle
\frac{\sum_{j=1}^{G}\ell_{ij}e_{ij}}
{\sum_{j=1}^{G}\ell_{ij}},
& \text{if } \sum_{j=1}^{G}\ell_{ij}>0,\\[10pt]
0,
& \text{otherwise}.
\end{cases}
\end{equation}
The evidence-guided reward adjustment $\ell_{ij}(e_{ij}-q_i)$ measures the relative evidence quality of the $j$-th response compared with the average evidence quality among all label-correct responses. Specifically, a positive value indicates that the response provides above-average evidence, while a negative value indicates inferior evidence quality.
Therefore, EGRR redistributes rewards from responses with below-average evidence quality to those with stronger evidence while preserving the original label-level learning signal. Responses with evidence quality higher than the group average receive additional rewards, whereas responses with inferior evidence are penalized.
The final reward is formulated as:
\begin{equation}
\begin{aligned}
R^{(i,j)}
=&\;
\lambda_{\mathrm{format}}
R_{\mathrm{format}}^{(i,j)}
+
\left(1-\lambda_{\mathrm{format}}\right)
\cdot
\ell_{ij}
\left(
1+
\mathbb{I}[y_i=\mathrm{fake}]
(e_{ij}-q_i)
\right).
\end{aligned}
\label{eq:egrr_reward}
\end{equation}

\begin{figure*}[t]
    \centering
    \includegraphics[width=\linewidth]{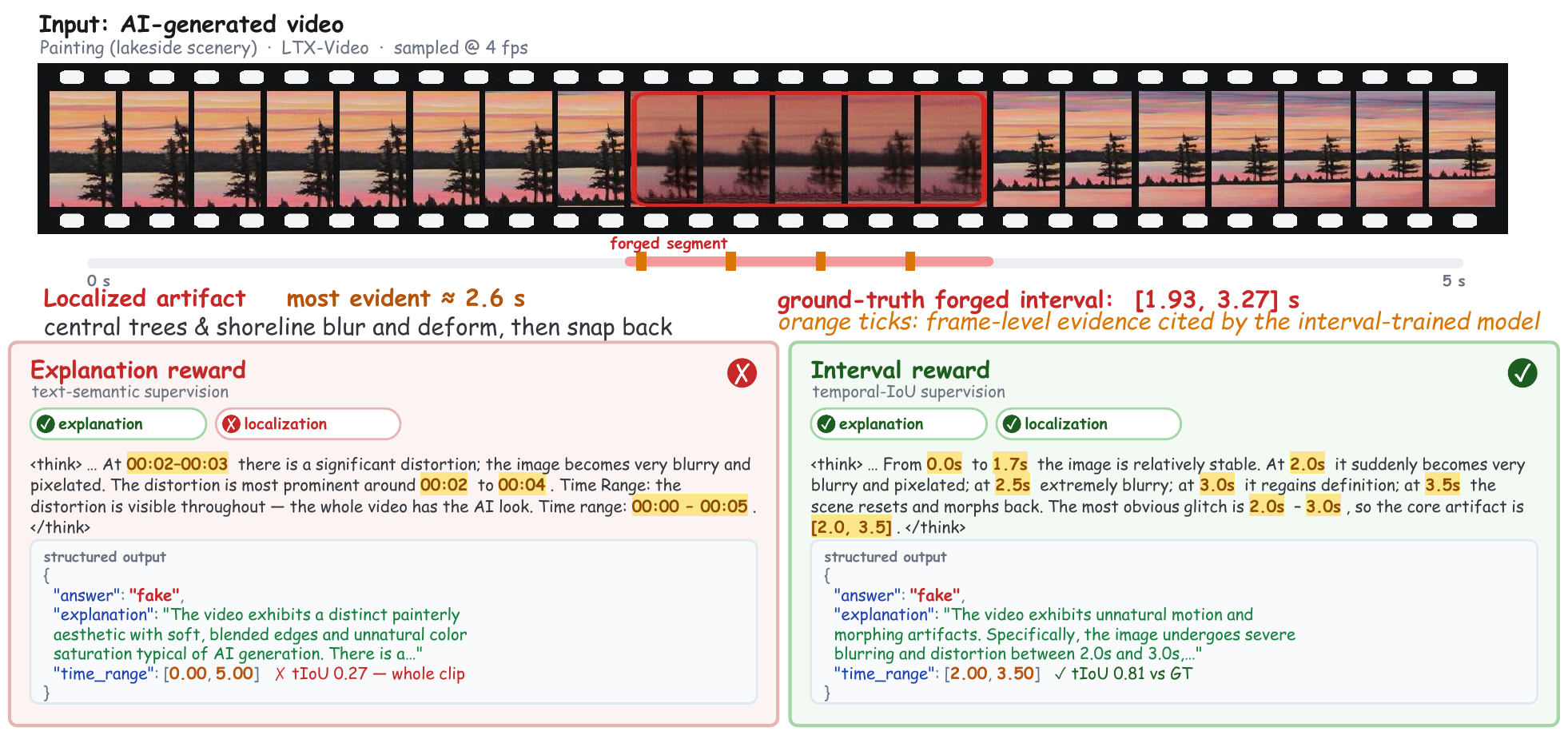}
    \vspace{-3mm}
    \caption{Comparison of detector outputs trained with two types of meta-detection feedback signals. The left and right columns show results from textual explanation-based and temporal grounding-based training, respectively. The detector trained with textual explanations fails to generalize to temporal localization, whereas the detector trained with temporal grounding produces both accurate temporal localization and reasonable textual explanations.}
    \label{fig4}
    \vspace{-1mm}
\end{figure*}

\subsection{Label-Preserving Evidence Calibration}
Let $\mathcal{C}_i=\{j\mid \ell_{ij}=1\}$ denote the set of label-correct
responses for video $V_i$, and let $M_i=|\mathcal{C}_i|$. For $M_i>0$,
$q_i$ is the empirical mean of evidence quality within $\mathcal{C}_i$.
The centered evidence residual therefore satisfies the following exact
finite-sample identity:
\begin{equation}
\sum_{j=1}^{G}\ell_{ij}(e_{ij}-q_i)
=
\sum_{j\in\mathcal{C}_i}(e_{ij}-q_i)
=
0.
\label{eq:zero_sum_evidence}
\end{equation}
To make this property explicit, we denote the semantic component of the EGRR
reward as
$r_{ij}^{\mathrm{sem}}
=
\ell_{ij}\bigl(1+\mathbb{I}[y_i=\mathrm{fake}](e_{ij}-q_i)\bigr)$.
Its average over label-correct responses remains unchanged:
\begin{equation}
\frac{1}{M_i}
\sum_{j\in\mathcal{C}_i}
r_{ij}^{\mathrm{sem}}
=
1+
\frac{\mathbb{I}[y_i=\mathrm{fake}]}{M_i}
\sum_{j\in\mathcal{C}_i}(e_{ij}-q_i)
=
1.
\label{eq:mean_reward_preservation}
\end{equation}
EGRR thus preserves the original semantic reward mass and only modifies its
allocation within the label-correct subset. Moreover, for any two
label-correct responses $j,k\in\mathcal{C}_i$ from the same fake video,
$r_{ij}^{\mathrm{sem}}-r_{ik}^{\mathrm{sem}}=e_{ij}-e_{ik}$.
The redistributed reward therefore preserves the ordering induced by evidence
quality. Stronger evidence receives higher credit without changing the average
label-level reward.
This conservation property also handles the asymmetric availability of
forgery evidence between real and fake videos. For real videos, the evidence
term is removed by $\mathbb{I}[y_i=\mathrm{fake}]$, so EGRR exactly retains the
original label reward. For fake videos, evidence quality changes only the
relative ranking within $\mathcal{C}_i$, while the average semantic reward
remains $1$. Hence, the availability of forgery evidence does not introduce a
systematic shift in the reward scale between real and fake samples. EGRR also
adapts to the informativeness of the sampled evidence. When
$\operatorname{Var}_{j\in\mathcal{C}_i}(e_{ij})=0$, all label-correct responses
satisfy $e_{ij}=q_i$, and every evidence residual becomes zero. EGRR then
reduces exactly to label-level RLVR. This includes the binary cases where all
evidence is incorrect or all evidence is correct. Therefore, EGRR introduces
evidence-aware refinement only when evidence quality provides additional
discriminative information.

\section{Experiments}
\subsection{Experiment Setup}

\begin{table*}[t]
\centering

\caption{
Performance comparison on ViF-Bench with Accuracy, Recall, and F1 score.
}

\footnotesize
\setlength{\tabcolsep}{1.7pt}
\renewcommand{\arraystretch}{1.06}

\begin{adjustbox}{max width=\textwidth,center}
\begin{tabular}{@{}ll*{20}{c}@{}}

\toprule

\multirow{2}{*}{\textbf{Method}}
&
\multirow{2}{*}{\textbf{Metric}}
&
\multicolumn{1}{c}{\textbf{\shortstack{Wan2.1\\-1.3B}}}
&
\multicolumn{1}{c}{\textbf{\shortstack{CogV\\-X1.5}}}
&
\multicolumn{2}{c}{\textbf{\shortstack{Wan2.2\\-5B}}}
&
\multicolumn{2}{c}{\textbf{\shortstack{Hunyuan\\Video}}}
&
\multicolumn{1}{c}{\textbf{\shortstack{VACE\\-1.3B}}}
&
\multicolumn{2}{c}{\textbf{\shortstack{Wan2.2\\-14B}}}
&
\multicolumn{2}{c}{\textbf{\shortstack{SkyReels\\-V2}}}
&
\multicolumn{2}{c}{\textbf{\shortstack{LTX-Video\\-13B}}}
&
\multicolumn{1}{c}{\textbf{\shortstack{Gen4\\-Turbo}}}
&
\multicolumn{1}{c}{\textbf{\shortstack{Hailuo\\-02}}}
&
\multicolumn{1}{c}{\textbf{\shortstack{Pika\\-V2}}}
&
\multicolumn{1}{c}{\textbf{\shortstack{PixVerse\\-V4.5}}}
&
\multicolumn{1}{c}{\textbf{\shortstack{Kling\\-V1}}}
&
\multicolumn{1}{c}{\textbf{Sora-2}}
&
\multirow{2}{*}{\textbf{Avg.}}
\\

\cmidrule(lr){3-21}

&
&
\multicolumn{2}{c}{\textbf{T2V}}
&
\textbf{T2V}
&
\textbf{I2V}
&
\textbf{T2V}
&
\textbf{I2V}
&
\textbf{T2V}
&
\textbf{T2V}
&
\textbf{I2V}
&
\textbf{T2V}
&
\textbf{I2V}
&
\textbf{T2V}
&
\textbf{I2V}
&
\multicolumn{6}{c}{\textbf{T2V}}
&
\\

\midrule


\multirow{3}{*}{Qwen3.7-Plus}
& Acc
& 61.04 & 67.48 & 61.04 & 53.37 & 57.72 & 52.45 & 51.23
& 56.13 & 52.76 & 58.28 & 53.09 & 57.50 & 53.07 & 54.95
& 54.41 & 67.33 & 77.33 & 55.00 & 55.74 & 57.89
\\
& R
& 27.61 & 40.49 & 27.61 & 12.27 & 20.99 & 10.43 & 7.98
& 17.79 & 11.04 & 22.09 & 11.73 & 20.00 & 11.66 & 17.12
& 14.71 & 40.00 & 60.67 & 15.71 & 17.57 & 21.44
\\
& F1
& 41.47 & 55.46 & 41.47 & 20.83 & 33.17 & 17.99 & 14.05
& 28.86 & 18.95 & 34.62 & 20.00 & 32.00 & 19.90 & 27.54
& 24.39 & 55.05 & 72.80 & 25.88 & 28.42 & 32.25
\\

\midrule
\multirow{3}{*}{GPT-5.5}
& Acc
& 56.13 & 58.28 & 54.29 & 51.53 & 57.72 & 50.92 & 50.31
& 57.06 & 50.61 & 56.44 & 51.23 & 58.75 & 52.45 & 50.90
& 53.31 & 79.33 & 78.67 & 51.79 & 50.68 & 56.34
\\
& R
& 13.50 & 17.79 & 9.82 & 4.29 & 16.67 & 3.07 & 1.84
& 15.34 & 2.45 & 14.11 & 3.70 & 18.75 & 6.13 & 3.60
& 8.09 & 59.33 & 58.67 & 5.00 & 2.70 & 13.94
\\
& F1
& 23.53 & 29.90 & 17.68 & 8.14 & 28.27 & 5.88 & 3.57
& 26.32 & 4.73 & 24.47 & 7.06 & 31.25 & 11.43 & 6.84
& 14.77 & 74.17 & 73.33 & 9.40 & 5.19 & 21.36
\\

\midrule
\midrule

\multirow{3}{*}{DeepTraceReward}
& Acc
& 52.45 & 49.69 & 52.45 & 51.23 & 52.78 & 48.77 & 52.45
& 51.84 & 50.31 & 51.84 & 51.23 & 52.19 & 50.61 & 51.35
& 52.57 & 48.67 & 51.67 & 51.07 & 51.35 & 51.29
\\
& R
& 98.77 & 93.25 & 98.77 & 96.32 & 99.38 & 91.41 & 98.77
& 97.55 & 94.48 & 97.55 & 96.30 & 98.12 & 95.09 & 95.50
& 97.79 & 90.67 & 97.33 & 95.00 & 96.62 & 96.25
\\
& F1
& 67.51 & 64.96 & 67.51 & 66.38 & 67.79 & 64.09 & 67.51
& 66.95 & 65.53 & 66.95 & 66.38 & 67.24 & 65.82 & 66.25
& 67.34 & 63.85 & 66.82 & 66.00 & 66.51 & 66.39
\\

\midrule

\multirow{3}{*}{BusterX++}
& Acc
& 58.28 & 61.96 & 55.52 & 51.84 & 59.26 & 50.61 & 58.28
& 56.13 & 48.16 & 61.04 & 55.25 & 68.12 & 54.60 & 55.86
& 56.99 & 62.33 & 65.33 & 56.43 & 52.36 & 57.28
\\
& R
& 31.29 & 38.65 & 25.77 & 18.40 & 33.33 & 15.95 & 31.29
& 26.99 & 11.04 & 36.81 & 25.31 & 51.25 & 23.93 & 27.93
& 28.68 & 39.33 & 45.33 & 28.57 & 20.95 & 29.52
\\
& F1
& 42.86 & 50.40 & 36.68 & 27.65 & 45.00 & 24.41 & 42.86
& 38.10 & 17.56 & 48.58 & 36.12 & 61.65 & 34.51 & 38.75
& 40.00 & 51.08 & 56.67 & 39.60 & 30.54 & 40.16
\\

\midrule
\midrule


\multirow{3}{*}{Qwen3.5-9B}
& Acc
& 49.70 & 44.24 & 46.06 & 45.15 & 45.12 & 43.33 & 45.76
& 45.15 & 44.85 & 46.06 & 42.38 & 51.56 & 42.94 & 43.30
& 45.62 & 51.99 & 54.93 & 45.04 & 42.67 & 46.10
\\
& R
& 9.70 & 6.67 & 7.88 & 1.82 & 6.10 & 0.61 & 4.24
& 6.06 & 3.03 & 9.09 & 3.05 & 13.75 & 2.45 & 0.89
& 5.11 & 17.22 & 23.68 & 3.55 & 0.67 & 6.61
\\
& F1
& 16.17 & 10.68 & 12.75 & 3.21 & 10.00 & 1.06 & 7.25
& 9.95 & 5.21 & 14.42 & 5.03 & 22.11 & 4.12 & 1.55
& 8.59 & 26.40 & 34.44 & 6.07 & 1.16 & 10.54
\\
\midrule
\multirow{3}{*}{\quad + Label-RL}
& Acc
& 76.44 & 76.98 & 70.45 & 71.07 & 69.33 & 68.57 & 73.50 & 70.68 & 70.09 & 73.66 & 74.28 & 71.05 & 71.69 & 72.30 & 74.66 & 75.82 & 75.14 & 72.87 & 69.49 & 72.53
\\
& R
& 65.98 & 67.07 & 64.22 & 65.18 & 61.36 & 60.78 & 64.16 & 59.52 & 60.36 & 61.95 & 62.81 & 63.33 & 62.48 & 63.63 & 61.84 & 67.87 & 66.54 & 65.29 & 60.61 & 63.42
\\
& F1
& 73.69 & 74.45 & 68.49 & 69.26 & 66.67 & 65.91 & 70.77
& 67.00 & 66.87 & 70.17 & 70.95 & 68.63 & 68.82 & 69.67
& 70.93 & 73.73 & 72.80 & 70.64 & 66.52 & 69.79
\\
\midrule
\multirow{3}{*}{\quad + Label-Exp-RL}
& Acc
& 79.74 & 78.90 & 77.36 & 77.98 & 75.51 & 76.24 & 76.94 & 74.88 & 74.34 & 76.66 & 75.69 & 78.50 & 77.85 & 73.92 & 75.05 & 80.54 & 79.32 & 72.30 & 73.30 & 76.58
\\
& R
& 69.01 & 70.36 & 66.39 & 65.71 & 68.31 & 69.03 & 66.63 & 63.94 & 63.34 & 66.96 & 67.82 & 65.87 & 65.03 & 65.07 & 67.93 & 70.55 & 69.70 & 64.23 & 63.02 & 66.78
\\
& F1
& 77.30 & 76.93 & 74.57 & 74.90 & 73.61 & 74.39 & 74.29
& 71.79 & 71.17 & 74.15 & 73.61 & 75.39 & 74.59 & 71.39
& 73.14 & 78.38 & 77.12 & 69.87 & 70.24 & 74.04
\\
\midrule

\rowcolor{oursgray}
& Acc
& 86.86 & 86.19 & 81.44 & 81.37 & 83.51 & 82.78 & 81.97 & 81.37 & 81.51 & 83.40 & 84.13 & 81.37 & 81.75 & 84.48 & 82.51 & 81.34 & 84.98 & 81.37 & 81.44 & 82.83
\\
\rowcolor{oursgray}
& R
& 76.53 & 75.98 & 74.54 & 73.88 & 72.94 & 72.00 & 73.66 & 71.25 & 71.30 & 71.43 & 72.27 & 74.31 & 75.01 & 72.96 & 71.19 & 71.10 & 77.15 & 71.20 & 71.25 & 73.16
\\
\rowcolor{oursgray}
\multirow{-3}{*}{\textbf{\quad + Label-Tem-RL}}
& F1
& 85.35 & 84.62 & 80.06 & 79.86 & 81.56 & 80.70 & 80.34
& 79.27 & 79.41 & 81.14 & 81.99 & 79.95 & 80.43 & 82.46
& 80.28 & 79.21 & 83.70 & 79.26 & 79.33 & 81.00
\\

\bottomrule
\end{tabular}
\end{adjustbox}
\vspace{-0.3cm}

\label{tab1}
\end{table*}

\paragraph{Implement Details.} 
We use Qwen3.5-9B \cite{qwen3.5} as the base detector and uniformly sample video frames at 4 FPS. To construct diverse fake videos, we employ open-source video generation models with different scales, including LTX-Video-2B \cite{HaCohen2024LTXVideo}, Wan2.2-Fun-5B-InP \cite{wan2025}, and SkyReels-V2-DF-14B-540P-Diffusers \cite{chen2025skyreelsv2infinitelengthfilmgenerative}, to generate 45K fake videos. In addition, we incorporate 5K high-fidelity challenging fake videos generated by closed-source models, including Wan2.7-I2V \cite{alibabacloud2026wan27i2v} and Seedance1.0-Pro \cite{byteplus2025seedance10pro}. Together with the corresponding real videos, this process results in a balanced dataset containing 100K samples.
For textual evidence construction, we leverage GPT-5.5 \cite{openai2026gpt55} to generate descriptions of observable forgery artifacts for each fake video. We further use Gemini-3.1-Pro \cite{googledeepmind2026gemini31pro} to filter out samples with unreliable textual explanations and remove their corresponding real videos to maintain a 1:1 ratio between real and fake samples. During evaluation, Qwen3.5-4B \cite{qwen3.5} is employed to measure the semantic consistency between detector-generated explanations and reference textual explanations.
To ensure a fair comparison between temporal evidence and textual explanations, we convert temporal grounding quality into a binary feedback signal using an IoU threshold of 0.7, consistent with the binary nature of textual evidence evaluation. We optimize the detector with the DAPO \cite{yu2026dapo} algorithm using a learning rate of $1\times10^{-6}$. All experiments are conducted on 16 NVIDIA H200-144GB GPUs for one training epoch.

\paragraph{Evaluation Metrics and Comparison Methods.} 
We compare our method with strong proprietary MLLMs, including Qwen3.7-Plus and GPT-5.5, as well as representative open-source MLLM-based detectors, including DeepTraceReward and BusterX++. We report Accuracy, F1 score, and Recall on ViF-Bench. Furthermore, we measure Recall on the fake-only subset of GenBuster-Bench to provide a focused evaluation of synthetic video identification capability, which is critical for reducing false negatives in real-world  scenarios.

\begin{table*}[t]
\centering

\caption{
Performance comparison on the fake-only subset of GenBuster-Bench, focusing on the capability of identifying AI-generated videos.
}

\footnotesize
\setlength{\tabcolsep}{3.2pt}
\renewcommand{\arraystretch}{1.06}

\begin{adjustbox}{max width=\textwidth,center}
\begin{tabular}{@{}l*{10}{c}@{}}
\toprule

\multirow{2}{*}{\textbf{Method}}
&
\multicolumn{9}{c}{\textbf{OOD (2025)}}
&
\multirow{2}{*}{\textbf{Wild (2026)}}
\\

\cmidrule(lr){2-10}

&
\textbf{Sora}
&
\textbf{Pika}
&
\textbf{Gen3}
&
\textbf{Luma}
&
\textbf{WanX}
&
\textbf{Kling}
&
\textbf{Jimeng}
&
\textbf{Vidu}
&
\textbf{Avg.}
&
\\

\midrule

Qwen3.7-Plus
& 75.5
& 93.0
& 89.0
& 99.0
& 93.3
& 86.0
& 75.0
& 88.7
& 87.4
& 64.7
\\
\methodspace

GPT-5.5
& 61.0
& 89.0
& 92.0
& 99.0
& 88.7
& 70.0
& 50.0
& 96.7
& 80.8
& 60.7
\\
\midrule

\methodspace

DeepTraceReward
& 83.0
& 90.0
& 89.0
& 91.0
& 85.3
& 86.0
& 79.0
& 93.3
& 87.1
& 90.0
\\

\methodspace

BusterX++
& 71.0
& 92.0
& 83.0
& 90.0
& 81.3
& 63.0
& 78.0
& 87.3
& 80.7
& 64.0
\\

\midrule

Qwen3.5-9B

& 36.0

& 60.0

& 66.0

& 77.0

& 62.7

& 53.0

& 46.0

& 58.7

& 57.4

& 26.7

\\

\quad + Label-RL
& 76.5\gain{40.5} & 90.0\gain{30.0} & 91.0\gain{25.0} & 86.0\gain{9.0} & 86.0\gain{23.3} & 78.0\gain{25.0} & 79.0\gain{33.0} & 80.7\gain{22.0} & 83.4\gain{26.0} & 75.0\gain{48.3}
\\

\methodspace

\quad + Label-Exp-RL
& 78.6\gain{42.6} & 91.6\gain{31.6} & 88.6\gain{22.6} & 93.3\gain{16.3} & 92.6\gain{29.9} & 82.6\gain{29.6} & 85.6\gain{39.6} & 83.9\gain{25.2} & 87.1\gain{29.7} & 82.8\gain{56.1}
\\

\methodspace

\rowcolor{oursgray}

\textbf{\quad + Label-Tem-RL}

& \best{89.7}\gain{53.7} & \best{95.8}\gain{35.8} & \best{93.6}\gain{27.6} & \best{99.5}\gain{22.5} & \best{95.0}\gain{32.3} & \best{90.0}\gain{37.0} & \best{87.2}\gain{41.2} & \best{97.3}\gain{38.6} & \best{93.5}\gain{36.1} & \best{94.2}\gain{67.5}

\\

\bottomrule
\end{tabular}
\end{adjustbox}

\label{tab2}
\end{table*}

\subsection{Main Results}

\paragraph{Analysis of Meta-Detection Effectiveness.}
To validate the effectiveness of meta-detection, we compare detectors trained with and without evidence-aware feedback. As shown in Tables~\ref{tab1} and~\ref{tab2}, as well as Fig.~\ref{fig5}, detectors trained with meta-detection consistently outperform conventional label-level reinforcement learning. For example, on ViF-Bench, Label-Tem-RL improves over Label-RL by 10.30\%, 9.74\%, and 11.21\% in Accuracy, Recall, and F1 score, respectively. On the fake-only subset of GenBuster-Bench, Label-Tem-RL further improves Recall by 10.1\% and 19.2\% on OOD and Wild evaluations, respectively. The radar chart in Fig.~\ref{fig5} further provides an intuitive visualization of these consistent improvements across different evaluation settings. These results demonstrate that evidence-aware reward redistribution among label-correct responses effectively encourages detectors to learn fine-grained forgery artifacts, leading to more robust and generalizable AI-generated video detection.

\begin{wrapfigure}{r}{0.5\textwidth}
\vspace{-4mm}
    \centering
    \includegraphics[width=0.5\textwidth]{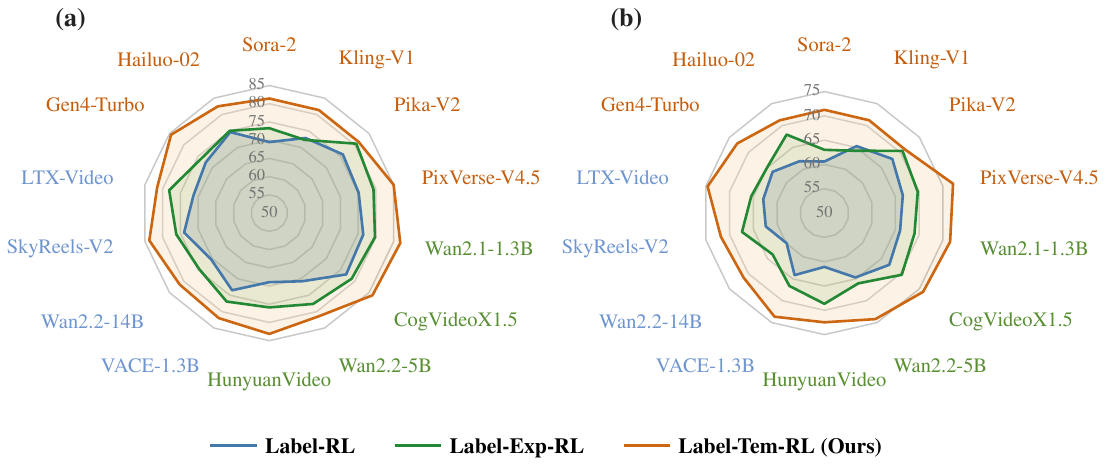}
    \vspace{-5mm}
    \caption{Performance on ViF-Bench. (a) Accuracy comparison. (b) Recall comparison.}
    \label{fig5}
    \vspace{-3mm}
\end{wrapfigure}

\paragraph{Temporal Grounding Outperforms Textual Explanations as Meta-Detection Feedback.}
To ensure a fair comparison between temporal grounding and textual explanations, we convert temporal evidence into binary feedback using an IoU threshold of 0.7, matching the binary nature of textual evidence evaluation. As shown in Tables~\ref{tab1} and~\ref{tab2}, detectors trained with temporal grounding consistently outperform those trained with textual explanations. On ViF-Bench, Label-Tem-RL improves over Label-Exp-RL by 6.25\%, 6.38\%, and 6.96\% in Accuracy, Recall, and F1 score, respectively. On the fake-only subset of GenBuster-Bench, temporal grounding achieves improvements of 6.4\% and 11.4\% on OOD and Wild evaluation, respectively. These results demonstrate that rule-based temporal grounding provides more reliable and generalizable evidence supervision for meta-detection than model-based textual explanations.
As shown in Fig.~\ref{fig4}, the model trained with textual supervision can describe general forgery artifacts but fails to localize the manipulated temporal interval. In contrast, temporal grounding supervision enables accurate time-range prediction while retaining the ability to generate meaningful textual explanations, further demonstrating the superiority of temporal grounding as a reliable evidence source for meta-detection.

\FloatBarrier

\section{Conclusion}
In this paper, we introduce the concept of \textbf{meta-detection} into AI-generated video detection, extending reinforcement learning beyond label correctness to jointly evaluate prediction correctness and evidence validity. We develop an automated and scalable data construction pipeline that generates paired real-fake videos while deriving ground-truth forgery temporal intervals from controlled manipulation processes. We further demonstrate that rule-based temporal grounding provides more reliable and verifiable supervision than model-based textual explanations, making it a more suitable evidence source for meta-detection. Moreover, we propose \textbf{Evidence-Guided Reward Redistribution (EGRR)}, which calibrates label-level rewards by redistributing credits among label-correct responses according to evidence quality, enabling effective integration of meta-detection feedback into reinforcement learning. Extensive experiments demonstrate that \textbf{VidForensics-M1} learns fine-grained visual forgery artifacts from temporal grounding evidence, leading to more robust and generalizable AI-generated video detection.

\bibliographystyle{unsrtnat}
\bibliography{references}

@inproceedings{zhang2026generative,
  title={Generative universal verifier as multimodal meta-reasoner},
  author={Zhang, Xinchen and Zhang, Xiaoying and Wu, Youbin and Cao, Yanbin and Zhang, Renrui and Chu, Ruihang and Yang, Ling and Yang, Yujiu and Shi, Guang},
  booktitle={International Conference on Learning Representations},
  volume={2026},
  pages={109211--109243},
  year={2026}
}

@article{zhang2026omniverifier,
  title={OmniVerifier-M1: Multimodal Meta-Verifier with Explicit Structured Recalibration},
  author={Zhang, Xinchen and Liu, Bowei and Liu, Jiale and Shi, Chufan and Zhang, Yizhen and Liu, Junhong and Zhang, Youliang and Li, Zhiheng and Yang, Yujiu and Yang, Ling},
  journal={arXiv preprint arXiv:2605.28805},
  year={2026}
}

@inproceedings{zhang2025itercomp,
  title={Itercomp: Iterative composition-aware feedback learning from model gallery for text-to-image generation},
  author={Zhang, Xinchen and Yang, Ling and Li, Guohao and Cai, Yaqi and Tang, Yong and Yang, Yujiu and Wang, Mengdi and CUI, Bin and others},
  booktitle={International Conference on Learning Representations},
  volume={2025},
  pages={31968--31988},
  year={2025}
}

@article{zhang2024realcompo,
  title={Realcompo: Balancing realism and compositionality improves text-to-image diffusion models},
  author={Zhang, Xinchen and Yang, Ling and Cai, Yaqi and Yu, Zhaochen and Wang, Kai-Ni and Xie, Jiake and Tian, Ye and Xu, Minkai and Tang, Yong and Yang, Yujiu and others},
  journal={Advances in Neural Information Processing Systems},
  volume={37},
  pages={96963--96992},
  year={2024}
}

@inproceedings{wang2020cnn,
  title={CNN-generated images are surprisingly easy to spot... for now},
  author={Wang, Sheng-Yu and Wang, Oliver and Zhang, Richard and Owens, Andrew and Efros, Alexei A},
  booktitle={Proceedings of the IEEE/CVF conference on computer vision and pattern recognition},
  pages={8695--8704},
  year={2020}
}

@inproceedings{xu2025hunyuanportrait,
  title={Hunyuanportrait: Implicit condition control for enhanced portrait animation},
  author={Xu, Zunnan and Yu, Zhentao and Zhou, Zixiang and Zhou, Jun and Jin, Xiaoyu and Hong, Fa-Ting and Ji, Xiaozhong and Zhu, Junwei and Cai, Chengfei and Tang, Shiyu and others},
  booktitle={2025 IEEE/CVF Conference on Computer Vision and Pattern Recognition (CVPR)},
  pages={15909--15919},
  year={2025},
  organization={IEEE}
}

@inproceedings{zhang2026zo3t,
  title={Zo3t: Zero-shot 3d-aware trajectory-guided image-to-video generation via test-time training},
  author={Zhang, Ruicheng and Zhou, Jun and Xu, Zunnan and Liu, Zihao and Huang, Jiehui and Zhang, Mingyang and Sun, Yu and Li, Xiu},
  booktitle={Proceedings of the AAAI Conference on Artificial Intelligence},
  volume={40},
  number={15},
  pages={12708--12716},
  year={2026}
}

@inproceedings{ojha2023universal,
  title={Towards universal fake image detectors that generalize across generative models},
  author={Ojha, Utkarsh and Li, Yuheng and Lee, Yong Jae},
  booktitle={Proceedings of the IEEE/CVF conference on computer vision and pattern recognition},
  pages={24480--24489},
  year={2023}
}

@inproceedings{hou2026survey,
  title={Detecting AI-Generated Video: A Vision--Language Dual-View Survey},
  author={Hou, Dylan Xinming and Zhang, Juntian and Gu, Xu and Wu, Yichen and Lukas, Nils and Xia, Gus and Chen, Xiuying and Liu, Yuhan},
  booktitle={Findings of the Association for Computational Linguistics: ACL 2026},
  pages={32221--32255},
  year={2026}
}

@article{gao2025davidxr1,
  title={DAVID-XR1: Detecting AI-Generated Videos with Explainable Reasoning},
  author={Gao, Yifeng and Ding, Yifan and Su, Hongyu and Li, Juncheng and Zhao, Yunhan and Luo, Lin and Chen, Zixing and Wang, Li and Wang, Xin and Wang, Yixu and others},
  journal={arXiv preprint arXiv:2506.14827},
  year={2025}
}

@article{fu2025deeptracereward,
  title={Learning Human-Perceived Fakeness in AI-Generated Videos via Multimodal LLMs},
  author={Fu, Xingyu and Liu, Siyi and Xu, Yinuo and Lu, Pan and Hu, Guangqiuse and Yang, Tianbo and Anantasagar, Taran and Shen, Christopher and Mao, Yikai and Liu, Yuanzhe and others},
  journal={arXiv preprint arXiv:2509.22646},
  year={2025}
}

@article{li2026skyra,
  title={Skyra: AI-Generated Video Detection via Grounded Artifact Reasoning},
  author={Li, Yifei and Zheng, Wenzhao and Zhang, Yanran and Sun, Runze and Zheng, Yu and Chen, Lei and Zhou, Jie and Lu, Jiwen},
  journal={arXiv preprint arXiv:2512.15693},
  year={2025}
}

@article{park2026vidguardr,
  title={Vidguard-r1: Ai-generated video detection and explanation via reasoning mllms and rl},
  author={Park, Kyoungjun and Yang, Yifan and Yi, Juheon and Zheng, Shicheng and Shen, Yifei and Han, Dongqi and Shan, Caihua and Muaz, Muhammad and Qiu, Lili},
  journal={arXiv preprint arXiv:2510.02282},
  year={2025}
}

@article{wen2025busterx,
  title={Busterx: Mllm-powered ai-generated video forgery detection and explanation},
  author={Wen, Haiquan and He, Yiwei and Huang, Zhenglin and Li, Tianxiao and Yu, Zihan and Huang, Xingru and Qi, Lu and Wu, Baoyuan and Li, Xiangtai and Cheng, Guangliang},
  journal={arXiv preprint arXiv:2505.12620},
  year={2025}
}

@article{wen2025busterxpp,
  title={Busterx++: Towards unified cross-modal ai-generated content detection and explanation with mllm},
  author={Wen, Haiquan and Li, Tianxiao and Huang, Zhenglin and He, Yiwei and Cheng, Guangliang},
  journal={arXiv preprint arXiv:2507.14632},
  year={2025}
}

@inproceedings{tan2026videoveritas,
  title={Videoveritas: Ai-generated video detection via perception pretext reinforcement learning},
  author={Tan, Hao and Shi, Senyuan and Tan, Zichang and Yu, Zijian and Zhu, Huijia and Wang, Weiqiang and Wan, Jun and Lei, Zhen and others},
  booktitle={Forty-third International Conference on Machine Learning},
  year={2026}
}

@inproceedings{bai2024gvd,
  title={Ai-generated video detection via spatial-temporal anomaly learning},
  author={Bai, Jianfa and Lin, Man and Cao, Gang and Lou, Zijie},
  booktitle={Chinese Conference on Pattern Recognition and Computer Vision (PRCV)},
  pages={460--470},
  year={2024},
  organization={Springer}
}

@article{chen2026genvideo,
  title={Demamba: Ai-generated video detection on million-scale genvideo benchmark},
  author={Chen, Haoxing and Hong, Yan and Huang, Zizheng and Xu, Zhuoer and Gu, Zhangxuan and Li, Yaohui and Lan, Jun and Zhu, Huijia and Zhang, Jianfu and Wang, Weiqiang and others},
  journal={arXiv preprint arXiv:2405.19707},
  year={2024}
}

@inproceedings{ma2025gvf,
  title={Detecting ai-generated video via frame consistency},
  author={Ma, Long and Yan, Zhiyuan and Guo, Qinglang and Liao, Yong and Yu, Haiyang and Zhou, Pengyuan},
  booktitle={2025 IEEE International Conference on Multimedia and Expo (ICME)},
  pages={1--6},
  year={2025},
  organization={IEEE}
}

@inproceedings{ni2026genvidbench,
  title={Genvidbench: A 6-million benchmark for ai-generated video detection},
  author={Ni, Zhenliang and Yan, Qiangyu and Huang, Mouxiao and Yuan, Tianning and Tang, Yehui and Hu, Hailin and Chen, Xinghao and Wang, Yunhe},
  booktitle={Proceedings of the AAAI Conference on Artificial Intelligence},
  volume={40},
  number={18},
  pages={15582--15590},
  year={2026}
}

@misc{qwen3.5,
    title  = {{Qwen3.5}: Towards Native Multimodal Agents},
    author = {{Qwen Team}},
    month  = {February},
    year   = {2026},
    howpublished = {\url{https://qwen.ai/blog?id=qwen3.5}}
}

@article{HaCohen2024LTXVideo,
  title={LTX-Video: Realtime Video Latent Diffusion},
  author={HaCohen, Yoav and Chiprut, Nisan and Brazowski, Benny and Shalem, Daniel and Moshe, Dudu and Richardson, Eitan and Levin, Eran and Shiran, Guy and Zabari, Nir and Gordon, Ori and Panet, Poriya and Weissbuch, Sapir and Kulikov, Victor and Bitterman, Yaki and Melumian, Zeev and Bibi, Ofir},
  journal={arXiv preprint arXiv:2501.00103},
  year={2024}
}

@article{wan2025,
      title={Wan: Open and Advanced Large-Scale Video Generative Models}, 
      author={Team Wan and Ang Wang and Baole Ai and Bin Wen and Chaojie Mao and Chen-Wei Xie and Di Chen and Feiwu Yu and Haiming Zhao and Jianxiao Yang and Jianyuan Zeng and Jiayu Wang and Jingfeng Zhang and Jingren Zhou and Jinkai Wang and Jixuan Chen and Kai Zhu and Kang Zhao and Keyu Yan and Lianghua Huang and Mengyang Feng and Ningyi Zhang and Pandeng Li and Pingyu Wu and Ruihang Chu and Ruili Feng and Shiwei Zhang and Siyang Sun and Tao Fang and Tianxing Wang and Tianyi Gui and Tingyu Weng and Tong Shen and Wei Lin and Wei Wang and Wei Wang and Wenmeng Zhou and Wente Wang and Wenting Shen and Wenyuan Yu and Xianzhong Shi and Xiaoming Huang and Xin Xu and Yan Kou and Yangyu Lv and Yifei Li and Yijing Liu and Yiming Wang and Yingya Zhang and Yitong Huang and Yong Li and You Wu and Yu Liu and Yulin Pan and Yun Zheng and Yuntao Hong and Yupeng Shi and Yutong Feng and Zeyinzi Jiang and Zhen Han and Zhi-Fan Wu and Ziyu Liu},
      journal = {arXiv preprint arXiv:2503.20314},
      year={2025}
}

@misc{chen2025skyreelsv2infinitelengthfilmgenerative,
      title={SkyReels-V2: Infinite-length Film Generative Model}, 
      author={Guibin Chen and Dixuan Lin and Jiangping Yang and Chunze Lin and Junchen Zhu and Mingyuan Fan and Hao Zhang and Sheng Chen and Zheng Chen and Chengcheng Ma and Weiming Xiong and Wei Wang and Nuo Pang and Kang Kang and Zhiheng Xu and Yuzhe Jin and Yupeng Liang and Yubing Song and Peng Zhao and Boyuan Xu and Di Qiu and Debang Li and Zhengcong Fei and Yang Li and Yahui Zhou},
      year={2025},
      eprint={2504.13074},
      archivePrefix={arXiv},
      primaryClass={cs.CV},
      url={https://arxiv.org/abs/2504.13074}, 
}

@misc{alibabacloud2026wan27i2v,
  author       = {{Alibaba Cloud}},
  title        = {{Wan 2.7}: Image-to-Video API},
  year         = {2026},
  howpublished = {\url{https://www.alibabacloud.com/help/en/model-studio/image-to-video-general-api-reference}}
}

@misc{byteplus2025seedance10pro,
  author       = {{BytePlus}},
  title        = {{Seedance 1.0 Pro}},
  year         = {2025},
  howpublished = {\url{https://docs.byteplus.com/en/docs/ModelArk/1587798}},
}

@misc{openai2026gpt55,
    title        = {Introducing GPT-5},
    author       = {{OpenAI}},
    year         = {2025},
    howpublished = {\url{https://openai.com/index/introducing-gpt-5/}}
}

@misc{googledeepmind2026gemini31pro,
  author      = {{Google DeepMind}},
  title       = {{Gemini 3.1 Pro} Model Card},
  year        = {2026},
  howpublished = {\url{https://deepmind.google/models/model-cards/gemini-3-1-pro/}},
}

@article{yu2026dapo,
  title={Dapo: An open-source llm reinforcement learning system at scale},
  author={Yu, Qiying and Zhang, Zheng and Zhu, Ruofei and Yuan, Yufeng and Zuo, Xiaochen and Yue, Yu and Dai, Weinan and Fan, Tiantian and Liu, Gaohong and Liu, Lingjun and others},
  journal={Advances in Neural Information Processing Systems},
  volume={38},
  pages={113222--113244},
  year={2026}
}

@inproceedings{yan2024transcending,
  title={Transcending forgery specificity with latent space augmentation for generalizable deepfake detection},
  author={Yan, Zhiyuan and Luo, Yuhao and Lyu, Siwei and Liu, Qingshan and Wu, Baoyuan},
  booktitle={Proceedings of the IEEE/CVF Conference on Computer Vision and Pattern Recognition},
  pages={8984--8994},
  year={2024}
}

@inproceedings{tan2024rethinking,
  title={Rethinking the up-sampling operations in cnn-based generative network for generalizable deepfake detection},
  author={Tan, Chuangchuang and Zhao, Yao and Wei, Shikui and Gu, Guanghua and Liu, Ping and Wei, Yunchao},
  booktitle={Proceedings of the IEEE/CVF conference on computer vision and pattern recognition},
  pages={28130--28139},
  year={2024}
}

@inproceedings{nguyen2024laa,
  title={Laa-net: Localized artifact attention network for quality-agnostic and generalizable deepfake detection},
  author={Nguyen, Dat and Mejri, Nesryne and Singh, Inder Pal and Kuleshova, Polina and Astrid, Marcella and Kacem, Anis and Ghorbel, Enjie and Aouada, Djamila},
  booktitle={Proceedings of the IEEE/CVF Conference on Computer Vision and Pattern Recognition},
  pages={17395--17405},
  year={2024}
}

@inproceedings{fu2025exploring,
  title={Exploring unbiased deepfake detection via token-level shuffling and mixing},
  author={Fu, Xinghe and Yan, Zhiyuan and Yao, Taiping and Chen, Shen and Li, Xi},
  booktitle={Proceedings of the AAAI Conference on Artificial Intelligence},
  volume={39},
  number={3},
  pages={3040--3048},
  year={2025}
}

@inproceedings{yan2025sanity,
  title={A sanity check for ai-generated image detection},
  author={Yan, Shilin and Li, Ouxiang and Cai, Jiayin and Hao, Yanbin and Jiang, Xiaolong and Hu, Yao and Xie, Weidi},
  booktitle={International Conference on Learning Representations},
  volume={2025},
  pages={70702--70720},
  year={2025}
}

@article{yang2025all,
  title={All patches matter, more patches better: Enhance ai-generated image detection via panoptic patch learning},
  author={Yang, Zheng and Chen, Ruoxin and Yan, Zhiyuan and Zhang, Ke-Yue and Fu, Xinghe and Wu, Shuang and Shu, Xiujun and Yao, Taiping and Ding, Shouhong and Qin, Zequn and others},
  journal={arXiv preprint arXiv:2504.01396},
  year={2025}
}

@article{fu2025learning,
  title={Learning Human-Perceived Fakeness in AI-Generated Videos via Multimodal LLMs},
  author={Fu, Xingyu and Liu, Siyi and Xu, Yinuo and Lu, Pan and Hu, Guangqiuse and Yang, Tianbo and Anantasagar, Taran and Shen, Christopher and Mao, Yikai and Liu, Yuanzhe and others},
  journal={arXiv preprint arXiv:2509.22646},
  year={2025}
}

@article{davodi2026perceptual,
  title={Perceptual Judgments of Video Authenticity: An Examination of Viewing Duration, Confidence, Content, and Strategies},
  author={Davodi, Catherine E and Barrington, Sarah and Farid, Hany and Cooper, Emily A},
  journal={Law Review},
  volume={107},
  pages={1753--1819},
  year={2026}
}

@article{vaccari2020deepfakes,
  title={Deepfakes and disinformation: Exploring the impact of synthetic political video on deception, uncertainty, and trust in news},
  author={Vaccari, Cristian and Chadwick, Andrew},
  journal={Social media+ society},
  volume={6},
  number={1},
  pages={2056305120903408},
  year={2020},
  publisher={SAGE Publications Sage UK: London, England}
}

@article{chandra2024reducing,
  title={Reducing risks posed by synthetic content an overview of technical approaches to digital content transparency},
  author={Chandra, Bilva and Dunietz, Jesse and Roberts, Kathleen and Lee, Yooyoung and Fontana, Peter and Awad, George},
  year={2024},
  publisher={Bilva Chandra, Jesse Dunietz, Kathleen Roberts, Yooyoung Lee, Peter Fontana~…}
}

@article{corvi2025seeing,
  title={Seeing What Matters: Generalizable AI-generated Video Detection with Forensic-Oriented Augmentation},
  author={Corvi, Riccardo and Cozzolino, Davide and Prashnani, Ekta and De Mello, Shalini and Nagano, Koki and Verdoliva, Luisa},
  journal={arXiv preprint arXiv:2506.16802},
  year={2025}
}

@article{turpin2023language,
  title={Language models don't always say what they think: Unfaithful explanations in chain-of-thought prompting},
  author={Turpin, Miles and Michael, Julian and Perez, Ethan and Bowman, Samuel},
  journal={Advances in Neural Information Processing Systems},
  volume={36},
  pages={74952--74965},
  year={2023}
}

@inproceedings{he2021forgerynet,
  title={Forgerynet: A versatile benchmark for comprehensive forgery analysis},
  author={He, Yinan and Gan, Bei and Chen, Siyu and Zhou, Yichun and Yin, Guojun and Song, Luchuan and Sheng, Lu and Shao, Jing and Liu, Ziwei},
  booktitle={Proceedings of the IEEE/CVF conference on computer vision and pattern recognition},
  pages={4360--4369},
  year={2021}
}

@article{li2026cubecomposer,
  title={CubeComposer: Spatio-Temporal Autoregressive 4K 360° Video Generation from Perspective Video},
  author={Li, Lingen and Wang, Guangzhi and Li, Xiaoyu and Zhang, Zhaoyang and Dou, Qi and Gu, Jinwei and Xue, Tianfan and Shan, Ying},
  journal={arXiv e-prints},
  pages={arXiv--2603},
  year={2026}
}

@article{xiong2026evatok,
  title={Evatok: Adaptive length video tokenization for efficient visual autoregressive generation},
  author={Xiong, Tianwei and Liew, Jun Hao and Huang, Zilong and Lin, Zhijie and Feng, Jiashi and Liu, Xihui},
  journal={arXiv preprint arXiv:2603.12267},
  year={2026}
}

@article{xu2025smrabooth,
  title={SMRABooth: Subject and Motion Representation Alignment for Customized Video Generation},
  author={Xu, Xuancheng and Li, Yaning and You, Sisi and Bao, Bing-Kun},
  journal={arXiv preprint arXiv:2512.12193},
  year={2025}
}

@inproceedings{NEURIPS2022_9d560961,
 author = {Wei, Jason and Wang, Xuezhi and Schuurmans, Dale and Bosma, Maarten and ichter, brian and Xia, Fei and Chi, Ed and Le, Quoc V and Zhou, Denny},
 booktitle = {Advances in Neural Information Processing Systems},
 doi = {10.52202/068431-1800},
 editor = {S. Koyejo and S. Mohamed and A. Agarwal and D. Belgrave and K. Cho and A. Oh},
 pages = {24824--24837},
 publisher = {Curran Associates, Inc.},
 title = {Chain-of-Thought Prompting Elicits Reasoning in Large Language Models},
 url = {https://proceedings.neurips.cc/paper_files/paper/2022/file/9d5609613524ecf4f15af0f7b31abca4-Paper-Conference.pdf},
 volume = {35},
 year = {2022}
}

@inproceedings{3294996.3295184,
author = {Christiano, Paul F. and Leike, Jan and Brown, Tom B. and Martic, Miljan and Legg, Shane and Amodei, Dario},
title = {Deep reinforcement learning from human preferences},
year = {2017},
isbn = {9781510860964},
publisher = {Curran Associates Inc.},
address = {Red Hook, NY, USA},
booktitle = {Proceedings of the 31st International Conference on Neural Information Processing Systems},
pages = {4302–4310},
numpages = {9},
location = {Long Beach, California, USA},
series = {NIPS'17}
}

@misc{amodei2016concreteproblemsaisafety,
      title={Concrete Problems in AI Safety}, 
      author={Dario Amodei and Chris Olah and Jacob Steinhardt and Paul Christiano and John Schulman and Dan Mané},
      year={2016},
      eprint={1606.06565},
      archivePrefix={arXiv},
      primaryClass={cs.AI},
      url={https://arxiv.org/abs/1606.06565}, 
}

@inproceedings{wang-etal-2025-grounded,
    title = "Grounded-{V}ideo{LLM}: Sharpening Fine-grained Temporal Grounding in Video Large Language Models",
    author = "Wang, Haibo  and
      Xu, Zhiyang  and
      Cheng, Yu  and
      Diao, Shizhe  and
      Zhou, Yufan  and
      Cao, Yixin  and
      Wang, Qifan  and
      Ge, Weifeng  and
      Huang, Lifu",
    editor = "Christodoulopoulos, Christos  and
      Chakraborty, Tanmoy  and
      Rose, Carolyn  and
      Peng, Violet",
    booktitle = "Findings of the Association for Computational Linguistics: EMNLP 2025",
    month = nov,
    year = "2025",
    address = "Suzhou, China",
    publisher = "Association for Computational Linguistics",
    url = "https://aclanthology.org/2025.findings-emnlp.50/",
    doi = "10.18653/v1/2025.findings-emnlp.50",
    pages = "959--975",
    ISBN = "979-8-89176-335-7"
}

@article{Guo_2025,
   title={DeepSeek-R1 incentivizes reasoning in LLMs through reinforcement learning},
   volume={645},
   ISSN={1476-4687},
   url={http://dx.doi.org/10.1038/s41586-025-09422-z},
   DOI={10.1038/s41586-025-09422-z},
   number={8081},
   journal={Nature},
   publisher={Springer Science and Business Media LLC},
   author={Guo, Daya and Yang, Dejian and Zhang, Haowei and Song, Junxiao and Wang, Peiyi and Zhu, Qihao and Xu, Runxin and Zhang, Ruoyu and Ma, Shirong and Bi, Xiao and Zhang, Xiaokang and Yu, Xingkai and Wu, Yu and Wu, Z. F. and Gou, Zhibin and Shao, Zhihong and Li, Zhuoshu and Gao, Ziyi and Liu, Aixin and Xue, Bing and Wang, Bingxuan and Wu, Bochao and Feng, Bei and Lu, Chengda and Zhao, Chenggang and Deng, Chengqi and Ruan, Chong and Dai, Damai and Chen, Deli and Ji, Dongjie and Li, Erhang and Lin, Fangyun and Dai, Fucong and Luo, Fuli and Hao, Guangbo and Chen, Guanting and Li, Guowei and Zhang, H. and Xu, Hanwei and Ding, Honghui and Gao, Huazuo and Qu, Hui and Li, Hui and Guo, Jianzhong and Li, Jiashi and Chen, Jingchang and Yuan, Jingyang and Tu, Jinhao and Qiu, Junjie and Li, Junlong and Cai, J. L. and Ni, Jiaqi and Liang, Jian and Chen, Jin and Dong, Kai and Hu, Kai and You, Kaichao and Gao, Kaige and Guan, Kang and Huang, Kexin and Yu, Kuai and Wang, Lean and Zhang, Lecong and Zhao, Liang and Wang, Litong and Zhang, Liyue and Xu, Lei and Xia, Leyi and Zhang, Mingchuan and Zhang, Minghua and Tang, Minghui and Zhou, Mingxu and Li, Meng and Wang, Miaojun and Li, Mingming and Tian, Ning and Huang, Panpan and Zhang, Peng and Wang, Qiancheng and Chen, Qinyu and Du, Qiushi and Ge, Ruiqi and Zhang, Ruisong and Pan, Ruizhe and Wang, Runji and Chen, R. J. and Jin, R. L. and Chen, Ruyi and Lu, Shanghao and Zhou, Shangyan and Chen, Shanhuang and Ye, Shengfeng and Wang, Shiyu and Yu, Shuiping and Zhou, Shunfeng and Pan, Shuting and Li, S. S. and Zhou, Shuang and Wu, Shaoqing and Yun, Tao and Pei, Tian and Sun, Tianyu and Wang, T. and Zeng, Wangding and Liu, Wen and Liang, Wenfeng and Gao, Wenjun and Yu, Wenqin and Zhang, Wentao and Xiao, W. L. and An, Wei and Liu, Xiaodong and Wang, Xiaohan and Chen, Xiaokang and Nie, Xiaotao and Cheng, Xin and Liu, Xin and Xie, Xin and Liu, Xingchao and Yang, Xinyu and Li, Xinyuan and Su, Xuecheng and Lin, Xuheng and Li, X. Q. and Jin, Xiangyue and Shen, Xiaojin and Chen, Xiaosha and Sun, Xiaowen and Wang, Xiaoxiang and Song, Xinnan and Zhou, Xinyi and Wang, Xianzu and Shan, Xinxia and Li, Y. K. and Wang, Y. Q. and Wei, Y. X. and Zhang, Yang and Xu, Yanhong and Li, Yao and Zhao, Yao and Sun, Yaofeng and Wang, Yaohui and Yu, Yi and Zhang, Yichao and Shi, Yifan and Xiong, Yiliang and He, Ying and Piao, Yishi and Wang, Yisong and Tan, Yixuan and Ma, Yiyang and Liu, Yiyuan and Guo, Yongqiang and Ou, Yuan and Wang, Yuduan and Gong, Yue and Zou, Yuheng and He, Yujia and Xiong, Yunfan and Luo, Yuxiang and You, Yuxiang and Liu, Yuxuan and Zhou, Yuyang and Zhu, Y. X. and Huang, Yanping and Li, Yaohui and Zheng, Yi and Zhu, Yuchen and Ma, Yunxian and Tang, Ying and Zha, Yukun and Yan, Yuting and Ren, Z. Z. and Ren, Zehui and Sha, Zhangli and Fu, Zhe and Xu, Zhean and Xie, Zhenda and Zhang, Zhengyan and Hao, Zhewen and Ma, Zhicheng and Yan, Zhigang and Wu, Zhiyu and Gu, Zihui and Zhu, Zijia and Liu, Zijun and Li, Zilin and Xie, Ziwei and Song, Ziyang and Pan, Zizheng and Huang, Zhen and Xu, Zhipeng and Zhang, Zhongyu and Zhang, Zhen},
   year={2025},
   month=sep, pages={633–638} }

\newpage

\appendix

\section{Theoretical Analysis of Evidence-Guided Reward Redistribution}
\label{app:egrr_theory}

This section provides a theoretical analysis of Evidence-Guided Reward
Redistribution (EGRR). We first show that EGRR exactly conserves the
label-level reward mass. We then prove that it preserves correctness as the
primary objective while introducing an orthogonal preference for evidence
quality. Finally, we characterize its policy-gradient effect and compare it
with naive multiplicative reward coupling.

\subsection{Preliminaries}

For a training video $V_i$, the policy samples a group of $G$ responses
$\{o_{ij}\}_{j=1}^{G}$. Let
\begin{equation}
    \ell_{ij}=\mathbb{I}[\hat{y}_{ij}=y_i]\in\{0,1\}
\end{equation}
denote label correctness, and let $e_{ij}\in[0,1]$ denote evidence quality.
Continuous temporal IoU and binary evidence correctness are both covered by
this definition. We further define
\begin{equation}
    z_i=\mathbb{I}[y_i=\mathrm{fake}],\qquad
    \mathcal{C}_i=\{j:\ell_{ij}=1\},\qquad
    M_i=|\mathcal{C}_i|.
\end{equation}
The group-relative evidence baseline is
\begin{equation}
q_i=
\begin{cases}
\displaystyle
\frac{\sum_{j=1}^{G}\ell_{ij}e_{ij}}
     {\sum_{j=1}^{G}\ell_{ij}}, & M_i>0,\\[8pt]
0, & M_i=0.
\end{cases}
\label{eq:app_q}
\end{equation}
The semantic component of the EGRR reward is
\begin{equation}
    r_{ij}^{\mathrm{EGRR}}
    =\ell_{ij}\bigl[1+z_i(e_{ij}-q_i)\bigr]
    =\ell_{ij}+\Delta_{ij},
\label{eq:app_egrr_semantic}
\end{equation}
where
\begin{equation}
    \Delta_{ij}=z_i\ell_{ij}(e_{ij}-q_i)
\label{eq:app_adjustment}
\end{equation}
is the evidence-guided reward adjustment. The complete reward is
\begin{equation}
    R_{ij}^{\mathrm{EGRR}}
    =\lambda_{\mathrm{format}}R_{ij}^{\mathrm{format}}
    +(1-\lambda_{\mathrm{format}})r_{ij}^{\mathrm{EGRR}}.
\label{eq:app_full_reward}
\end{equation}
Since the format term is unchanged from label-level RLVR, the following
analysis focuses on the semantic component in
Eq.~\eqref{eq:app_egrr_semantic}.

\subsection{Reward Conservation and Correctness Preservation}

\begin{lemma}[Exact finite-sample reward conservation]
\label{lem:egrr_conservation}
For every sampled group, the EGRR adjustment is zero-sum:
\begin{equation}
    \sum_{j=1}^{G}\Delta_{ij}=0.
\label{eq:app_zero_sum}
\end{equation}
Consequently,
\begin{equation}
    \sum_{j=1}^{G}r_{ij}^{\mathrm{EGRR}}
    =\sum_{j=1}^{G}\ell_{ij}=M_i.
\label{eq:app_mass_conservation}
\end{equation}
If $M_i>0$, the average semantic reward among label-correct responses is
exactly one:
\begin{equation}
    \frac{1}{M_i}\sum_{j\in\mathcal{C}_i}
    r_{ij}^{\mathrm{EGRR}}=1.
\label{eq:app_correct_mean}
\end{equation}
\end{lemma}

\begin{proof}
When $M_i=0$, every $\ell_{ij}$ is zero. Equations
\eqref{eq:app_zero_sum}--\eqref{eq:app_mass_conservation} then hold
immediately. When $M_i>0$, Eq.~\eqref{eq:app_q} gives
\begin{align}
    \sum_{j=1}^{G}\Delta_{ij}
    &=z_i\sum_{j\in\mathcal{C}_i}(e_{ij}-q_i) \\
    &=z_i\left(\sum_{j\in\mathcal{C}_i}e_{ij}-M_iq_i\right)=0.
\end{align}
Substituting this identity into Eq.~\eqref{eq:app_egrr_semantic} yields
Eq.~\eqref{eq:app_mass_conservation}. Restricting the same sum to
$\mathcal{C}_i$ and dividing by $M_i$ proves
Eq.~\eqref{eq:app_correct_mean}.
\end{proof}

\begin{corollary}[Preservation of the group mean]
\label{cor:egrr_group_mean}
EGRR has exactly the same group-mean semantic reward as label-level RLVR:
\begin{equation}
    \frac{1}{G}\sum_{j=1}^{G}r_{ij}^{\mathrm{EGRR}}
    =\frac{1}{G}\sum_{j=1}^{G}\ell_{ij}
    =\frac{M_i}{G}.
\end{equation}
Therefore, the full group reward satisfies
\begin{equation}
    \sum_{j=1}^{G}R_{ij}^{\mathrm{EGRR}}
    =\lambda_{\mathrm{format}}
      \sum_{j=1}^{G}R_{ij}^{\mathrm{format}}
     +(1-\lambda_{\mathrm{format}})M_i,
\end{equation}
which is identical to the reward mass produced by label-level RLVR with the
same format term. Under exchangeable sampling, it also follows that
$\mathbb{E}[r_{ij}^{\mathrm{EGRR}}]=\mathbb{E}[\ell_{ij}]$ for every response
index $j$.
\end{corollary}

\begin{proof}
The group-mean identity follows directly from
Lemma~\ref{lem:egrr_conservation}. For exchangeable samples, all response
indices have the same expected reward. Taking expectations on both sides of
Eq.~\eqref{eq:app_mass_conservation} and dividing by $G$ completes the proof.
\end{proof}

\begin{lemma}[Uniqueness of the group-relative baseline]
\label{lem:egrr_unique_baseline}
Consider a fake video with $M_i>0$ and an affine evidence calibration
\begin{equation}
    r_{ij}(b)=\ell_{ij}[1+e_{ij}-b],
\end{equation}
where the same scalar baseline $b$ is used for all label-correct responses.
The reward-conservation constraint
$\sum_{j=1}^{G}r_{ij}(b)=M_i$ holds if and only if $b=q_i$. Moreover,
\begin{equation}
    q_i=\arg\min_{b\in\mathbb{R}}
    \sum_{j\in\mathcal{C}_i}(e_{ij}-b)^2.
\label{eq:app_least_squares_baseline}
\end{equation}
Thus, the EGRR baseline is the unique constant centering term that conserves
the label reward mass, and it is also the least-squares center of the sampled
evidence scores.
\end{lemma}

\begin{proof}
For a fake video,
\begin{equation}
    \sum_{j=1}^{G}r_{ij}(b)
    =M_i+\sum_{j\in\mathcal{C}_i}e_{ij}-M_ib.
\end{equation}
This quantity equals $M_i$ if and only if
$b=M_i^{-1}\sum_{j\in\mathcal{C}_i}e_{ij}=q_i$. For the second claim, define
$f(b)=\sum_{j\in\mathcal{C}_i}(e_{ij}-b)^2$. Then
\begin{equation}
    f'(b)=2M_ib-2\sum_{j\in\mathcal{C}_i}e_{ij},
    \qquad f''(b)=2M_i>0.
\end{equation}
Hence, $f$ is strictly convex and has the unique minimizer $b=q_i$.
\end{proof}

\begin{theorem}[Correctness dominance and bounded calibration]
\label{thm:egrr_bounds}
For every label-incorrect response, $r_{ij}^{\mathrm{EGRR}}=0$. For every
label-correct response from a real video,
$r_{ij}^{\mathrm{EGRR}}=1$. For a fake video with $M_i\geq1$, every
$j\in\mathcal{C}_i$ satisfies
\begin{equation}
    \frac{1}{M_i}
    \leq r_{ij}^{\mathrm{EGRR}}
    \leq 2-\frac{1}{M_i}.
\label{eq:app_reward_bounds}
\end{equation}
Therefore, every label-correct response receives a strictly positive semantic
reward, whereas every label-incorrect response receives zero semantic reward.
Evidence quality refines credit only after label correctness has been
established. In particular, for every label-correct $j$ and label-incorrect
$k$ from the same group,
\begin{equation}
    r_{ij}^{\mathrm{EGRR}}-r_{ik}^{\mathrm{EGRR}}
    \geq \frac{1}{M_i}>0.
\end{equation}
\end{theorem}

\begin{proof}
The first two claims follow directly from
Eq.~\eqref{eq:app_egrr_semantic}. Consider a fake video and a response
$j\in\mathcal{C}_i$. If $M_i=1$, then $q_i=e_{ij}$ and
$r_{ij}^{\mathrm{EGRR}}=1$, which agrees with
Eq.~\eqref{eq:app_reward_bounds}. If $M_i>1$, let
\begin{equation}
    \bar e_{i,-j}
    =\frac{1}{M_i-1}
      \sum_{k\in\mathcal{C}_i\setminus\{j\}}e_{ik}.
\end{equation}
Then
\begin{equation}
    e_{ij}-q_i
    =\frac{M_i-1}{M_i}(e_{ij}-\bar e_{i,-j}).
\end{equation}
Since both $e_{ij}$ and $\bar e_{i,-j}$ lie in $[0,1]$, the centered residual
lies in
$[-(M_i-1)/M_i,(M_i-1)/M_i]$. Adding one proves
Eq.~\eqref{eq:app_reward_bounds}.
\end{proof}

\subsection{Orthogonal Decomposition of Label and Evidence Signals}

\begin{theorem}[Orthogonal label--evidence decomposition]
\label{thm:egrr_orthogonal}
Let $\bm{\ell}_i=(\ell_{i1},\ldots,\ell_{iG})^\top$ and let
$\mathbf{1}\in\mathbb{R}^{G}$ be the all-ones vector. Define the centered
label-reward vector
\begin{equation}
    \bm{a}_i^{\mathrm{label}}
    =\bm{\ell}_i-\frac{M_i}{G}\mathbf{1},
\end{equation}
and the evidence-adjustment vector
\begin{equation}
    \bm{\delta}_i
    =(\Delta_{i1},\ldots,\Delta_{iG})^\top.
\end{equation}
The group-centered EGRR reward admits the exact decomposition
\begin{equation}
    \bm{a}_i^{\mathrm{EGRR}}
    =\bm{a}_i^{\mathrm{label}}+\bm{\delta}_i.
\label{eq:app_advantage_decomposition}
\end{equation}
Moreover,
\begin{equation}
    \langle\bm{\delta}_i,\mathbf{1}\rangle=0,
    \qquad
    \langle\bm{\delta}_i,\bm{a}_i^{\mathrm{label}}\rangle=0.
\label{eq:app_orthogonality}
\end{equation}
If
\begin{equation}
    \operatorname{Var}_{\mathcal{C}_i}(e)
    =\begin{cases}
    \displaystyle
    \frac{1}{M_i}\sum_{j\in\mathcal{C}_i}(e_{ij}-q_i)^2,
    &M_i>0,\\[6pt]
    0,&M_i=0,
    \end{cases}
\end{equation}
then the group-centered reward variance satisfies
\begin{equation}
\begin{aligned}
    \frac{1}{G}\|\bm{a}_i^{\mathrm{EGRR}}\|_2^2
    &=\frac{M_i}{G}\left(1-\frac{M_i}{G}\right)
      +z_i\frac{M_i}{G}\operatorname{Var}_{\mathcal{C}_i}(e).
\end{aligned}
\label{eq:app_variance_decomposition}
\end{equation}
Thus, EGRR leaves the between-class label component unchanged and adds an
orthogonal within-label-correct evidence component.
\end{theorem}

\begin{proof}
By Corollary~\ref{cor:egrr_group_mean}, both label-level RLVR and EGRR have
the same group mean $M_i/G$. Subtracting this mean from
Eq.~\eqref{eq:app_egrr_semantic} gives
Eq.~\eqref{eq:app_advantage_decomposition}. The first identity in
Eq.~\eqref{eq:app_orthogonality} is exactly the zero-sum property in
Lemma~\ref{lem:egrr_conservation}. Since $\Delta_{ij}=0$ whenever
$j\notin\mathcal{C}_i$, we further obtain
\begin{align}
    \langle\bm{\delta}_i,\bm{a}_i^{\mathrm{label}}\rangle
    &=\sum_{j\in\mathcal{C}_i}
      \Delta_{ij}\left(1-\frac{M_i}{G}\right)\\
    &=\left(1-\frac{M_i}{G}\right)
      \sum_{j=1}^{G}\Delta_{ij}=0.
\end{align}
Therefore, the two components are orthogonal. It follows that
\begin{equation}
    \|\bm{a}_i^{\mathrm{EGRR}}\|_2^2
    =\|\bm{a}_i^{\mathrm{label}}\|_2^2
     +\|\bm{\delta}_i\|_2^2.
\end{equation}
Direct calculation gives
\begin{equation}
    \frac{1}{G}\|\bm{a}_i^{\mathrm{label}}\|_2^2
    =\frac{M_i}{G}\left(1-\frac{M_i}{G}\right)
\end{equation}
and
\begin{equation}
    \frac{1}{G}\|\bm{\delta}_i\|_2^2
    =z_i\frac{M_i}{G}\operatorname{Var}_{\mathcal{C}_i}(e).
\end{equation}
Combining these two identities proves
Eq.~\eqref{eq:app_variance_decomposition}.
\end{proof}

\begin{corollary}[Adaptive fallback to label-level RLVR]
\label{cor:egrr_fallback}
For a fake video, EGRR reduces exactly to label-level RLVR if and only if
$\operatorname{Var}_{\mathcal{C}_i}(e)=0$, with the convention above for
$M_i=0$. Moreover,
\begin{equation}
    \sum_{j=1}^{G}
    \bigl(r_{ij}^{\mathrm{EGRR}}-\ell_{ij}\bigr)^2
    =z_iM_i\operatorname{Var}_{\mathcal{C}_i}(e)
    \leq \frac{z_iM_i}{4}.
\label{eq:app_adaptive_strength}
\end{equation}
Hence, EGRR introduces no evidence-based perturbation when the sampled
evidence is non-discriminative. Its calibration strength increases
automatically with evidence dispersion and remains bounded.
\end{corollary}

\begin{proof}
Equation~\eqref{eq:app_adaptive_strength} follows from the definition of
$\bm{\delta}_i$. A finite set has zero variance if and only if all its values
are identical. Thus, every centered residual is zero exactly when
$\operatorname{Var}_{\mathcal{C}_i}(e)=0$. Finally, any random variable
supported on $[0,1]$ has variance at most $1/4$, which proves the bound. For
binary evidence, the variance is $q_i(1-q_i)$. EGRR therefore reduces to
label-level RLVR when all label-correct responses have incorrect evidence
($q_i=0$) or when all have correct evidence ($q_i=1$).
\end{proof}

\paragraph{Relation to normalized group advantages.}
If a GRPO-style method uses
$\widetilde{\bm{a}}_i=(\bm{r}_i-\bar r_i\mathbf{1})/(s_i+\epsilon)$,
Eq.~\eqref{eq:app_advantage_decomposition} remains exact in the numerator.
Standard-deviation normalization applies only a common positive scalar to the
sum of the two components. It does not change reward conservation, evidence
ordering, or the orthogonality of the unnormalized components.

\subsection{Evidence Preference and Policy-Gradient Interpretation}

\begin{lemma}[Evidence-order preservation]
\label{lem:egrr_ordering}
For any two label-correct responses $j,k\in\mathcal{C}_i$ from the same fake
video,
\begin{equation}
    r_{ij}^{\mathrm{EGRR}}-r_{ik}^{\mathrm{EGRR}}
    =e_{ij}-e_{ik}.
\label{eq:app_order_preservation}
\end{equation}
Thus, EGRR induces exactly the same pairwise ordering and margin as the
evidence score. Label-level RLVR assigns zero margin to the same pair.
\end{lemma}

\begin{proof}
Both responses have $\ell_{ij}=\ell_{ik}=1$ and share the same baseline $q_i$.
Subtracting their rewards cancels both the unit label reward and $q_i$, which
yields Eq.~\eqref{eq:app_order_preservation}.
\end{proof}

\begin{lemma}[Pairwise form of the evidence gradient]
\label{lem:egrr_pairwise_gradient}
Let $\bm{g}_{ij}$ be any vector associated with response $o_{ij}$. For a fake
video with $M_i>0$,
\begin{equation}
\begin{aligned}
    \sum_{j\in\mathcal{C}_i}(e_{ij}-q_i)\bm{g}_{ij}
    =\frac{1}{M_i}
      \sum_{\substack{j<k\\j,k\in\mathcal{C}_i}}
      (e_{ij}-e_{ik})(\bm{g}_{ij}-\bm{g}_{ik}).
\end{aligned}
\label{eq:app_pairwise_gradient}
\end{equation}
When $\bm{g}_{ij}=\nabla_\theta\log\pi_\theta(o_{ij}\mid V_i)$, the EGRR
adjustment is therefore a sum of pairwise policy-gradient preferences between
label-correct responses.
\end{lemma}

\begin{proof}
Expanding the right-hand side of Eq.~\eqref{eq:app_pairwise_gradient} gives
\begin{align}
&\frac{1}{M_i}
  \sum_{\substack{j<k\\j,k\in\mathcal{C}_i}}
  (e_{ij}-e_{ik})(\bm{g}_{ij}-\bm{g}_{ik})\\
&\quad=\sum_{j\in\mathcal{C}_i}e_{ij}\bm{g}_{ij}
-\frac{1}{M_i}
 \left(\sum_{j\in\mathcal{C}_i}e_{ij}\right)
 \left(\sum_{j\in\mathcal{C}_i}\bm{g}_{ij}\right)\\
&\quad=\sum_{j\in\mathcal{C}_i}(e_{ij}-q_i)\bm{g}_{ij},
\end{align}
where the last equality uses Eq.~\eqref{eq:app_q}.
\end{proof}

\begin{theorem}[Expected evidence-improving gradient]
\label{thm:egrr_expected_gradient}
Fix a fake training video and sample $G$ responses independently from
$\pi_\theta$. Let $\ell(o)$ denote label correctness and define
\begin{equation}
    p_\theta=\Pr_{o\sim\pi_\theta}[\ell(o)=1],
    \qquad
    \mu_\theta=
    \mathbb{E}_{o\sim\pi_\theta}[e(o)\mid\ell(o)=1].
\end{equation}
Assume that $p_\theta>0$, that $e(o)$ and the correctness rule do not explicitly
depend on $\theta$, and that differentiation and expectation can be
interchanged. As in standard RLVR, all rule-based rewards and group statistics
are treated as stop-gradient quantities. Let
\begin{equation}
    \bm{H}_G
    =\sum_{j=1}^{G}
      \ell(o_j)(e(o_j)-q)\nabla_\theta\log\pi_\theta(o_j)
\label{eq:app_evidence_estimator}
\end{equation}
be the unnormalized evidence component of the policy-gradient estimator. Then
\begin{equation}
    \mathbb{E}[\bm{H}_G]
    =\kappa_G(p_\theta)\nabla_\theta\mu_\theta,
\label{eq:app_expected_evidence_gradient}
\end{equation}
where
\begin{equation}
    \kappa_G(p)
    =Gp-1+(1-p)^G
    =\sum_{m=2}^{G}(m-1)
      \binom{G}{m}p^m(1-p)^{G-m}\geq0.
\label{eq:app_kappa}
\end{equation}
For $G\geq2$ and $p_\theta>0$, $\kappa_G(p_\theta)>0$. Consequently,
\begin{equation}
    \left\langle
    \mathbb{E}[\bm{H}_G],\nabla_\theta\mu_\theta
    \right\rangle
    =\kappa_G(p_\theta)\|\nabla_\theta\mu_\theta\|_2^2\geq0.
\label{eq:app_ascent_alignment}
\end{equation}
Thus, before clipping and optional variance normalization, the expected EGRR
evidence update is an ascent direction for evidence quality conditioned on a
correct label. At the same time, Lemma~\ref{lem:egrr_conservation} guarantees
that this update does not change the sampled label-reward mass.
\end{theorem}

\begin{proof}
Let $M=\sum_{j=1}^{G}\ell(o_j)$. Conditioned on $M=m\geq1$, the $m$
label-correct responses are identically distributed according to
\begin{equation}
    \rho_\theta(o)
    =\pi_\theta(o\mid \ell(o)=1).
\end{equation}
Let $\bm{g}(o)=\nabla_\theta\log\pi_\theta(o)$. For $m$ independent samples
from $\rho_\theta$, direct expansion of the sample-centered sum gives
\begin{align}
&\mathbb{E}\left[
  \sum_{a=1}^{m}(e_a-\bar e)\bm{g}_a
  \;\middle|\; M=m\right]\\
&\qquad=(m-1)
\left(
\mathbb{E}_{\rho_\theta}[e\bm{g}]
-\mathbb{E}_{\rho_\theta}[e]
 \mathbb{E}_{\rho_\theta}[\bm{g}]
\right)\\
&\qquad=(m-1)\operatorname{Cov}_{\rho_\theta}(e,\bm{g}).
\label{eq:app_conditional_covariance}
\end{align}
We next differentiate the conditional mean
$\mu_\theta=\mathbb{E}_{\rho_\theta}[e]$. Since
\begin{equation}
    \nabla_\theta\log\rho_\theta(o)
    =\bm{g}(o)-\nabla_\theta\log p_\theta,
\end{equation}
we have
\begin{align}
    \nabla_\theta\mu_\theta
    &=\mathbb{E}_{\rho_\theta}
      \left[e\nabla_\theta\log\rho_\theta(o)\right]\\
    &=\mathbb{E}_{\rho_\theta}[e\bm{g}]
      -\mu_\theta\mathbb{E}_{\rho_\theta}[\bm{g}]\\
    &=\operatorname{Cov}_{\rho_\theta}(e,\bm{g}).
\label{eq:app_covariance_gradient}
\end{align}
Combining Eqs.~\eqref{eq:app_conditional_covariance} and
\eqref{eq:app_covariance_gradient} yields
\begin{equation}
    \mathbb{E}[\bm{H}_G\mid M=m]
    =(m-1)\nabla_\theta\mu_\theta
\end{equation}
for $m\geq1$. The estimator is zero when $m=0$. Since
$M\sim\operatorname{Binomial}(G,p_\theta)$,
\begin{align}
    \mathbb{E}[\bm{H}_G]
    &=\mathbb{E}[(M-1)\mathbb{I}[M\geq1]]
      \nabla_\theta\mu_\theta\\
    &=\left(\mathbb{E}[M]-\Pr[M\geq1]\right)
      \nabla_\theta\mu_\theta\\
    &=\left(Gp_\theta-1+(1-p_\theta)^G\right)
      \nabla_\theta\mu_\theta.
\end{align}
This proves Eqs.~\eqref{eq:app_expected_evidence_gradient} and
\eqref{eq:app_kappa}. Equation~\eqref{eq:app_ascent_alignment} follows by
taking the inner product with $\nabla_\theta\mu_\theta$.
\end{proof}

\begin{corollary}[Progressive activation]
\label{cor:egrr_progressive}
The evidence-gradient coefficient satisfies
\begin{equation}
    \kappa_G(p)
    =\binom{G}{2}p^2+\mathcal{O}(p^3)
    \quad\text{as }p\rightarrow0,
    \qquad
    \kappa_G(1)=G-1.
\end{equation}
Therefore, EGRR activates evidence optimization only when a group contains at
least two label-correct responses that can be compared. Label learning remains
available for every label-correct response through the conserved unit reward.
As label accuracy improves, the evidence-learning signal becomes increasingly
dense.
\end{corollary}

\begin{proof}
The expansion follows from the binomial series
$(1-p)^G=1-Gp+\binom{G}{2}p^2+\mathcal{O}(p^3)$. The value at $p=1$ follows
directly from Eq.~\eqref{eq:app_kappa}.
\end{proof}

\subsection{Comparison with Naive Multiplicative Coupling}

\begin{proposition}[Removal of evidence-induced label gating]
\label{prop:egrr_vs_multiplicative}
Consider the naive multiplicative semantic reward
\begin{equation}
    r_{ij}^{\mathrm{mult}}
    =\ell_{ij}\bigl[(1-z_i)+z_ie_{ij}\bigr].
\label{eq:app_multiplicative_reward}
\end{equation}
For a fake video with $M_i>0$, its total reward is
\begin{equation}
    \sum_{j=1}^{G}r_{ij}^{\mathrm{mult}}=M_iq_i,
\label{eq:app_multiplicative_mass}
\end{equation}
whereas EGRR always assigns total reward $M_i$. If all label-correct responses
have the same evidence quality $c\in[0,1]$, then
\begin{equation}
    r_{ij}^{\mathrm{mult}}=c\ell_{ij},
    \qquad
    r_{ij}^{\mathrm{EGRR}}=\ell_{ij}.
\end{equation}
In particular, when $c=0$, multiplicative coupling assigns zero reward to both
label-correct and label-incorrect responses and removes the label-learning
signal. EGRR reduces exactly to label-level RLVR. Thus, EGRR introduces
evidence preference without multiplicatively gating correctness supervision.
\end{proposition}

\begin{proof}
For a fake video, Eq.~\eqref{eq:app_multiplicative_reward} reduces to
$r_{ij}^{\mathrm{mult}}=\ell_{ij}e_{ij}$. Therefore,
\begin{equation}
    \sum_{j=1}^{G}r_{ij}^{\mathrm{mult}}
    =\sum_{j\in\mathcal{C}_i}e_{ij}=M_iq_i,
\end{equation}
which proves Eq.~\eqref{eq:app_multiplicative_mass}. If every label-correct
response has evidence quality $c$, then $q_i=c$. Substitution into
Eq.~\eqref{eq:app_egrr_semantic} gives
$r_{ij}^{\mathrm{EGRR}}=\ell_{ij}$, while the multiplicative reward remains
$c\ell_{ij}$.
\end{proof}

\paragraph{Summary.}
The above results establish five properties of EGRR. First, it exactly
conserves label-level reward mass for every sampled group. Second, it retains
strict reward separation between label-correct and label-incorrect responses.
Third, it adds an orthogonal within-correct evidence signal without shifting
the group baseline. Fourth, its expected evidence gradient improves the
conditional evidence quality whenever comparative evidence is available.
Finally, it falls back to label-level RLVR when evidence is non-discriminative,
avoiding the evidence-induced suppression suffered by multiplicative reward
coupling.

\section{Real-Fake Video Pair Examples}
\label{app:data-construction}

The figure presents examples of the constructed real-fake video pairs. For each video generation model, two examples are provided. Each example consists of three rows: (1) the original authentic video; (2) the boundary frames provided to the video generation model and the generated synthetic segment; and (3) the final forged video obtained by aligning the generated segment with the removed content in terms of frame rate and resolution before inserting it back into the original position. The AI-generated segments are highlighted with red bounding boxes. The generated segments are aligned with the removed segments in terms of duration, resolution, and frame rate, while the surrounding prefix and suffix frames are directly copied from the original videos. Since the manipulated intervals are explicitly specified during the construction process, the corresponding forgery boundaries can be automatically recorded, providing objective and auditable evidence.

\begin{figure*}[tbp]\centering
{\footnotesize Original real video}\\[2pt]
\includegraphics[width=0.92\linewidth]{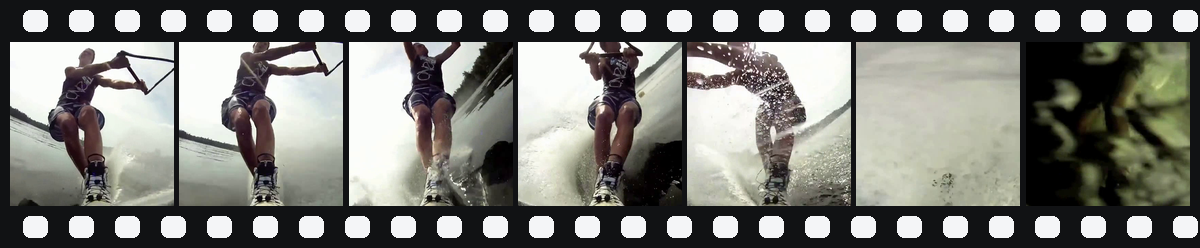}\\[6pt]
{\footnotesize Boundary (start/end) frames \;$\rightarrow$\; generative model \;$\rightarrow$\; generated segment}\\[2pt]
\raisebox{-0.5\height}{\includegraphics[height=1.7cm]{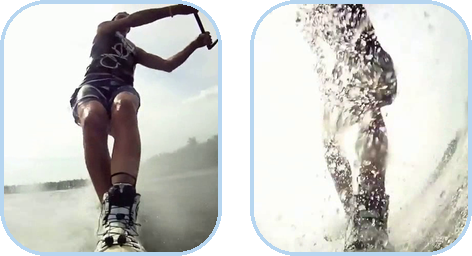}}\qquad
\raisebox{-0.5\height}{\includegraphics[height=1.7cm]{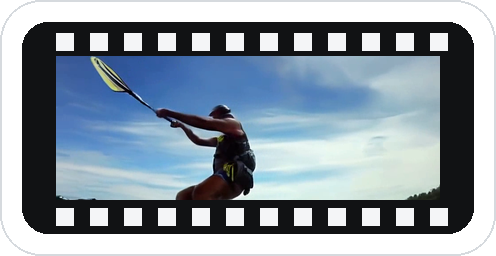}}\\[6pt]
{\footnotesize Constructed video after re-inserting the generated segment (boxed in red)}\\[2pt]
\includegraphics[width=0.92\linewidth]{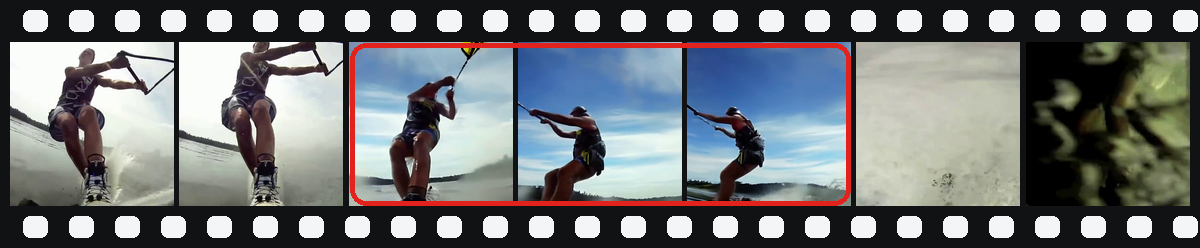}
\caption{Data construction with \textbf{Wan2.2-Fun-In}. Source video: ActivityNet (\emph{Waterskiing}); generation interval $[1.13, 3.40]$\,s of a $5$\,s clip.}
\label{fig:dg-wan22-waterskiing}
\end{figure*}

\begin{figure*}[tbp]\centering
{\footnotesize Original real video}\\[2pt]
\includegraphics[width=0.92\linewidth]{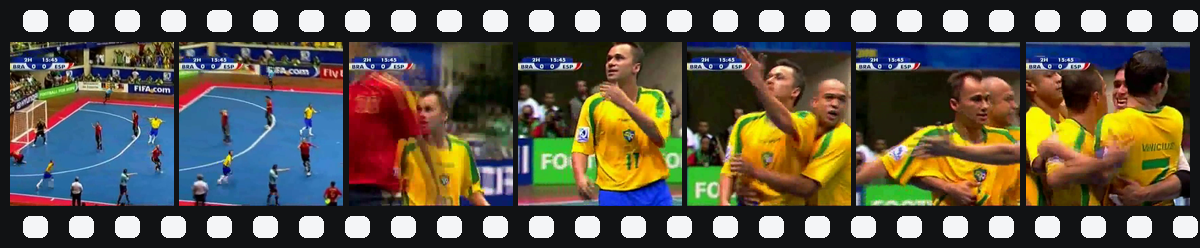}\\[6pt]
{\footnotesize Boundary (start/end) frames \;$\rightarrow$\; generative model \;$\rightarrow$\; generated segment}\\[2pt]
\raisebox{-0.5\height}{\includegraphics[height=1.7cm]{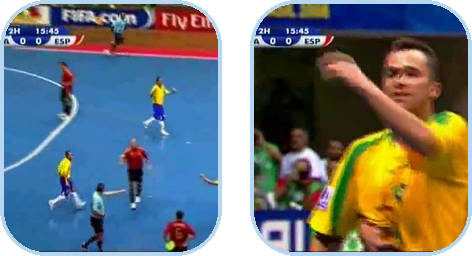}}\qquad
\raisebox{-0.5\height}{\includegraphics[height=1.7cm]{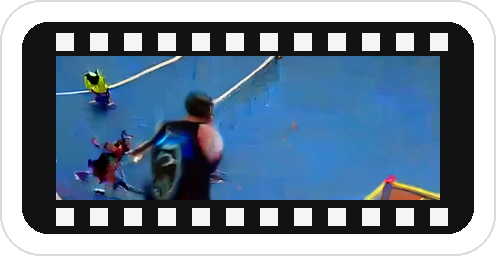}}\\[6pt]
{\footnotesize Constructed video after re-inserting the generated segment (boxed in red)}\\[2pt]
\includegraphics[width=0.92\linewidth]{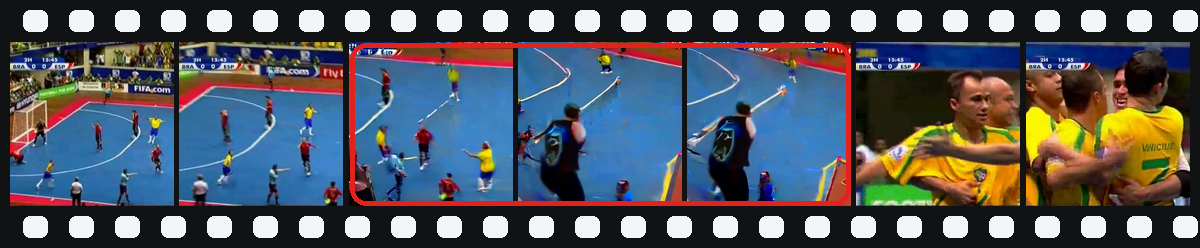}
\caption{Data construction with \textbf{Wan2.2-Fun-In}. Source video: ActivityNet (\emph{Futsal}); generation interval $[1.27, 3.53]$\,s of a $5$\,s clip.}
\label{fig:dg-wan22-futsal}
\end{figure*}

\begin{figure*}[tbp]\centering
{\footnotesize Original real video}\\[2pt]
\includegraphics[width=0.92\linewidth]{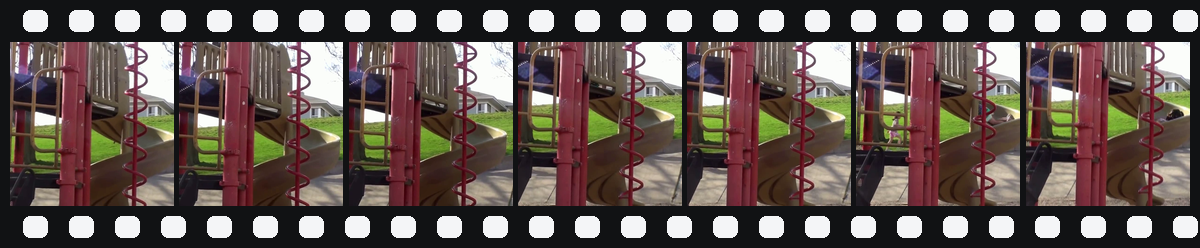}\\[6pt]
{\footnotesize Boundary (start/end) frames \;$\rightarrow$\; generative model \;$\rightarrow$\; generated segment}\\[2pt]
\raisebox{-0.5\height}{\includegraphics[height=1.7cm]{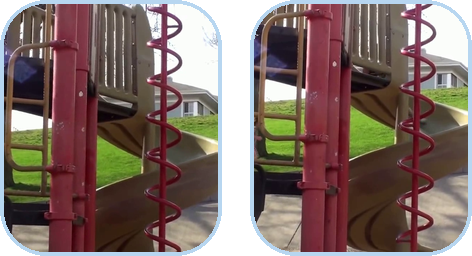}}\qquad
\raisebox{-0.5\height}{\includegraphics[height=1.7cm]{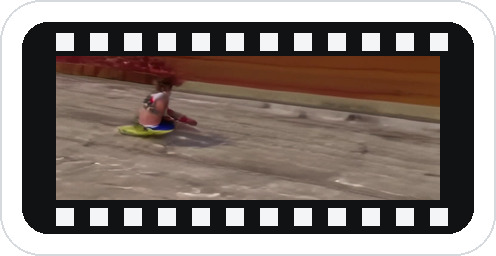}}\\[6pt]
{\footnotesize Constructed video after re-inserting the generated segment (boxed in red)}\\[2pt]
\includegraphics[width=0.92\linewidth]{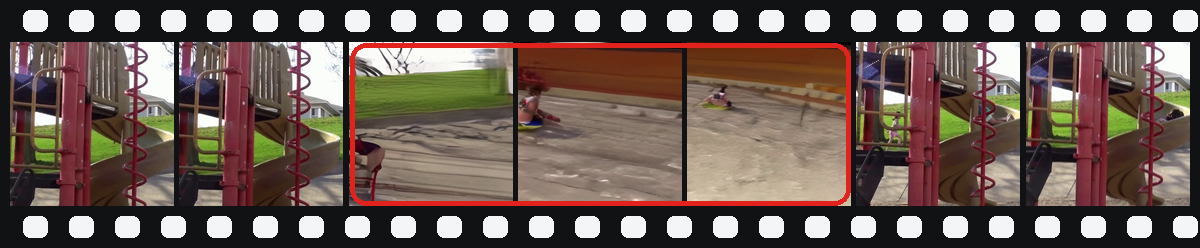}
\caption{Data construction with \textbf{LTX-Video}. Source video: ActivityNet (\emph{Fun sliding down}); generation interval $[0.73, 3.13]$\,s of a $5$\,s clip.}
\label{fig:dg-ltx-sliding}
\end{figure*}

\begin{figure*}[tbp]\centering
{\footnotesize Original real video}\\[2pt]
\includegraphics[width=0.92\linewidth]{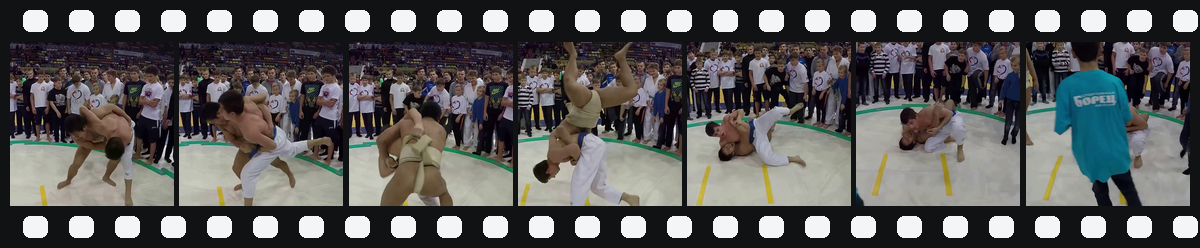}\\[6pt]
{\footnotesize Boundary (start/end) frames \;$\rightarrow$\; generative model \;$\rightarrow$\; generated segment}\\[2pt]
\raisebox{-0.5\height}{\includegraphics[height=1.7cm]{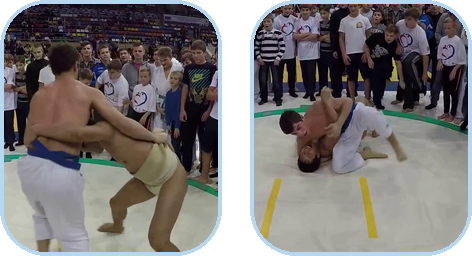}}\qquad
\raisebox{-0.5\height}{\includegraphics[height=1.7cm]{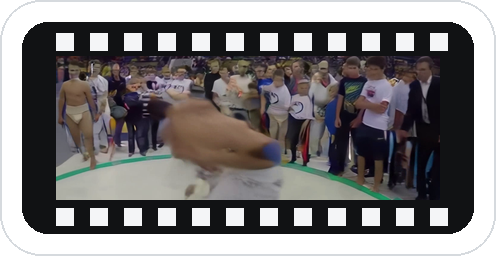}}\\[6pt]
{\footnotesize Constructed video after re-inserting the generated segment (boxed in red)}\\[2pt]
\includegraphics[width=0.92\linewidth]{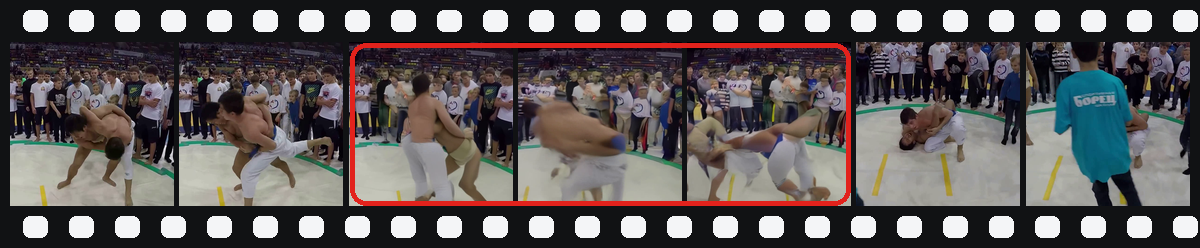}
\caption{Data construction with \textbf{LTX-Video}. Source video: ActivityNet (\emph{Sumo}); generation interval $[1.47, 3.87]$\,s of a $5$\,s clip.}
\label{fig:dg-ltx-sumo}
\end{figure*}

\begin{figure*}[tbp]\centering
{\footnotesize Original real video}\\[2pt]
\includegraphics[width=0.92\linewidth]{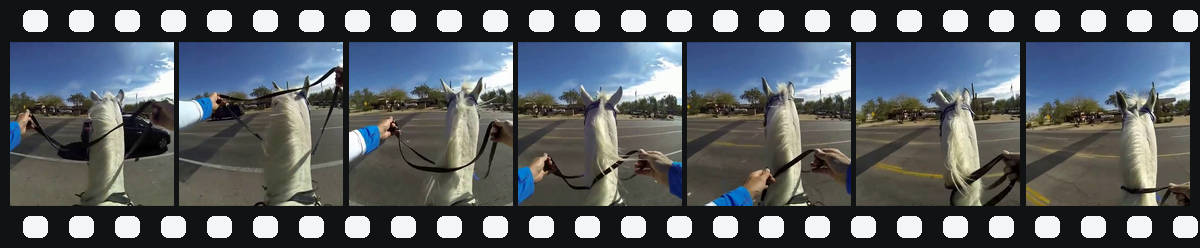}\\[6pt]
{\footnotesize First/last frames \;$\rightarrow$\; generative model \;$\rightarrow$\; generated segment}\\[2pt]
\raisebox{-0.5\height}{\includegraphics[height=1.7cm]{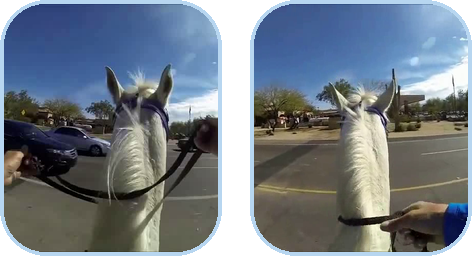}}\qquad
\raisebox{-0.5\height}{\includegraphics[height=1.7cm]{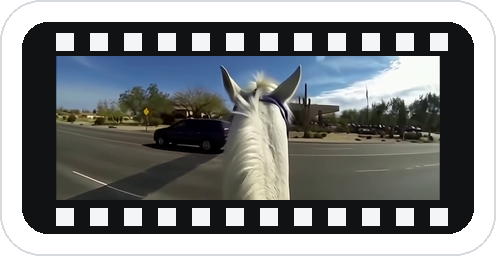}}\\[6pt]
{\footnotesize Constructed video after re-inserting the generated segment (boxed in red)}\\[2pt]
\includegraphics[width=0.92\linewidth]{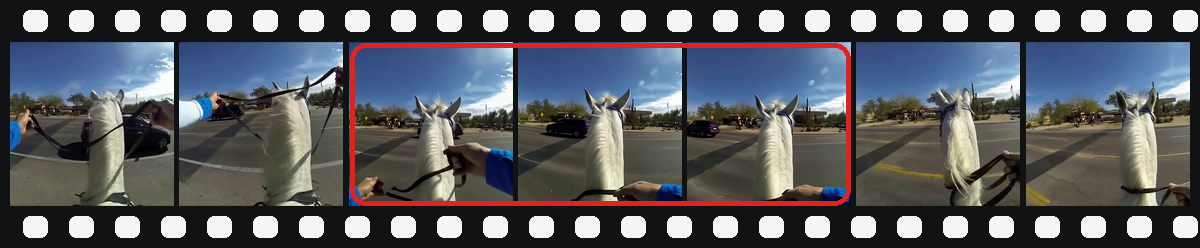}
\caption{Data construction with \textbf{Wan2.7-i2v}. Source video: ActivityNet (\emph{Horseback riding}); generation interval $[1.20, 3.40]$\,s of a $5$\,s clip.}
\label{fig:dg-wan27-horseback}
\end{figure*}

\begin{figure*}[tbp]\centering
{\footnotesize Original real video}\\[2pt]
\includegraphics[width=0.92\linewidth]{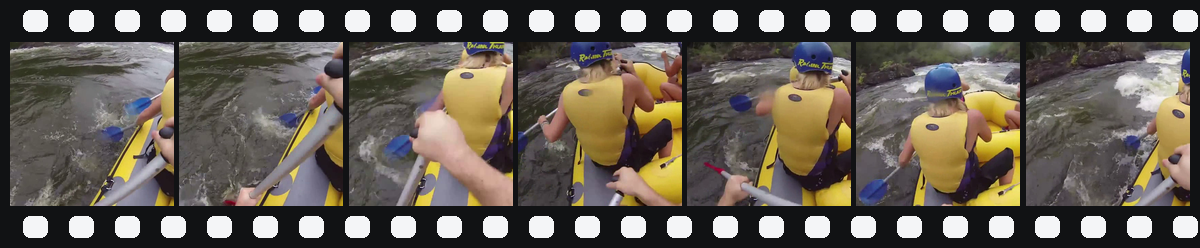}\\[6pt]
{\footnotesize First/last frames \;$\rightarrow$\; generative model \;$\rightarrow$\; generated segment}\\[2pt]
\raisebox{-0.5\height}{\includegraphics[height=1.7cm]{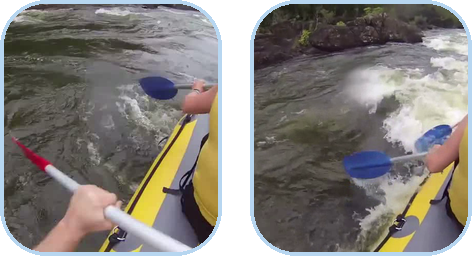}}\qquad
\raisebox{-0.5\height}{\includegraphics[height=1.7cm]{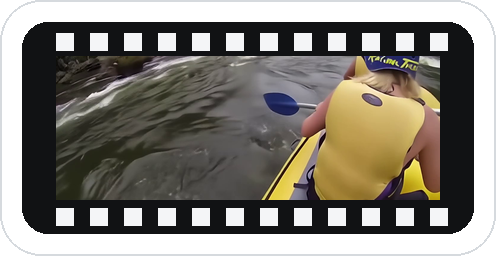}}\\[6pt]
{\footnotesize Constructed video after re-inserting the generated segment (boxed in red)}\\[2pt]
\includegraphics[width=0.92\linewidth]{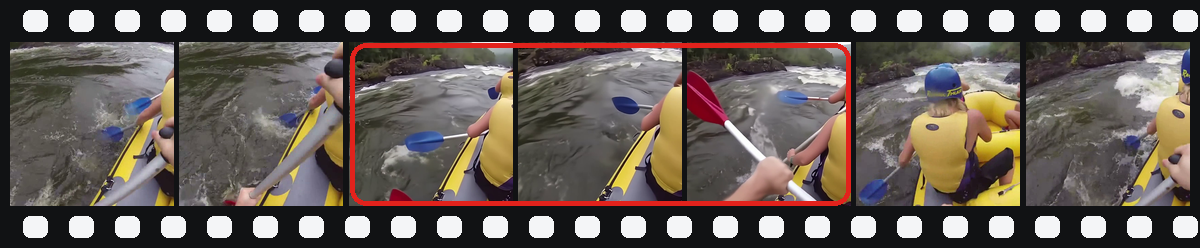}
\caption{Data construction with \textbf{Wan2.7-i2v}. Source video: ActivityNet (\emph{Rafting}); generation interval $[1.20, 3.40]$\,s of a $5$\,s clip.}
\label{fig:dg-wan27-rafting}
\end{figure*}

\begin{figure*}[tbp]\centering
{\footnotesize Original real video}\\[2pt]
\includegraphics[width=0.92\linewidth]{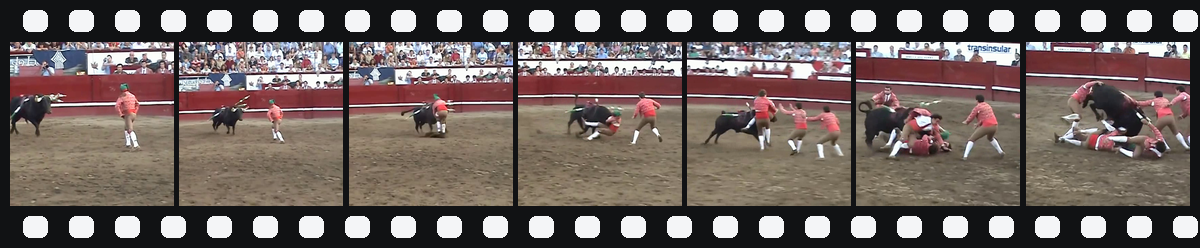}\\[6pt]
{\footnotesize First/last frames \;$\rightarrow$\; generative model \;$\rightarrow$\; generated segment}\\[2pt]
\raisebox{-0.5\height}{\includegraphics[height=1.7cm]{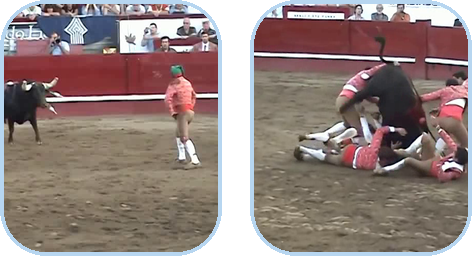}}\qquad
\raisebox{-0.5\height}{\includegraphics[height=1.7cm]{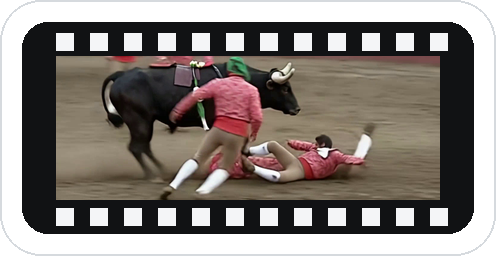}}\\[6pt]
{\footnotesize Constructed video after re-inserting the generated segment (boxed in red)}\\[2pt]
\includegraphics[width=0.92\linewidth]{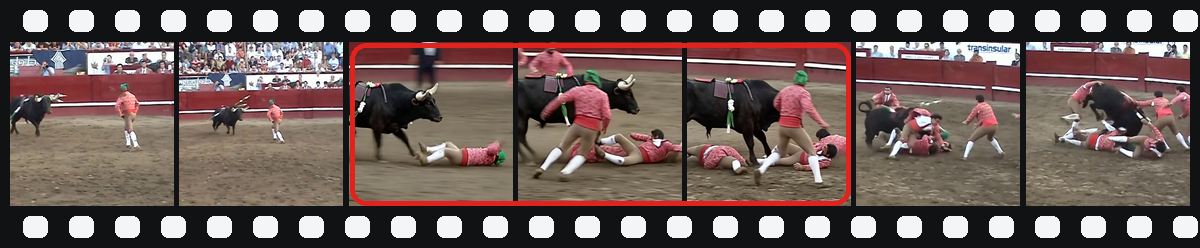}
\caption{Data construction with \textbf{Seedance 1.0 pro}. Source video: ActivityNet (\emph{Bullfighting}); generation interval $[1.20, 3.40]$\,s of a $5$\,s clip.}
\label{fig:dg-seedance-bullfight}
\end{figure*}

\begin{figure*}[tbp]\centering
{\footnotesize Original real video}\\[2pt]
\includegraphics[width=0.92\linewidth]{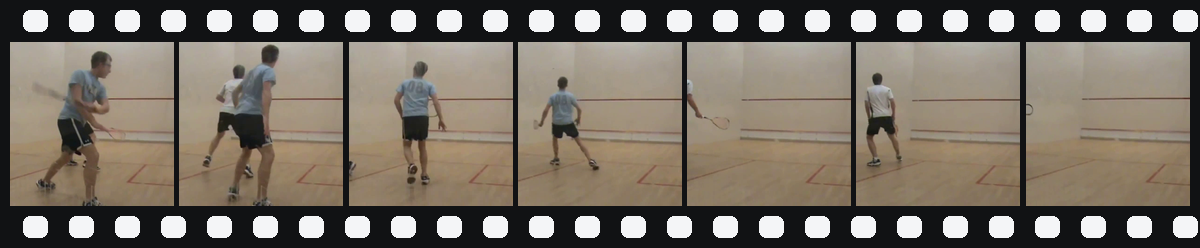}\\[6pt]
{\footnotesize First/last frames \;$\rightarrow$\; generative model \;$\rightarrow$\; generated segment}\\[2pt]
\raisebox{-0.5\height}{\includegraphics[height=1.7cm]{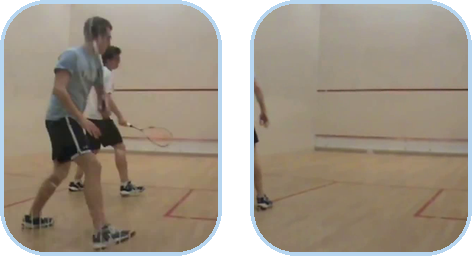}}\qquad
\raisebox{-0.5\height}{\includegraphics[height=1.7cm]{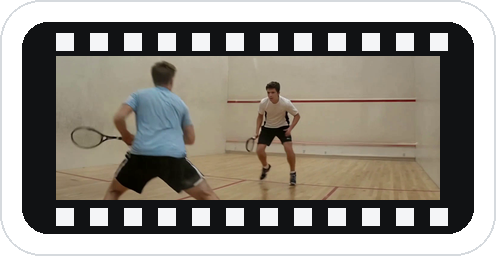}}\\[6pt]
{\footnotesize Constructed video after re-inserting the generated segment (boxed in red)}\\[2pt]
\includegraphics[width=0.92\linewidth]{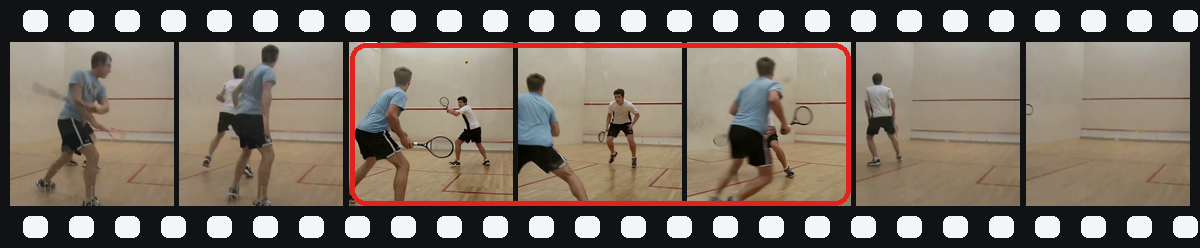}
\caption{Data construction with \textbf{Seedance 1.0 pro}. Source video: ActivityNet (\emph{Playing squash}); generation interval $[1.20, 3.40]$\,s of a $5$\,s clip.}
\label{fig:dg-seedance-squash}
\end{figure*}

\begin{figure*}[tbp]\centering
{\footnotesize Original real video}\\[2pt]
\includegraphics[width=0.92\linewidth]{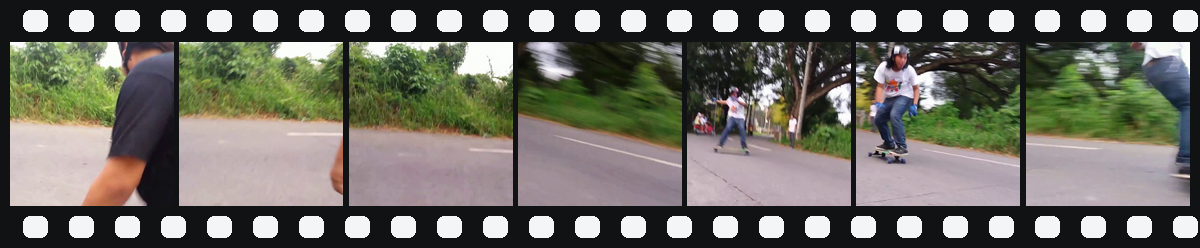}\\[6pt]
{\footnotesize First/last frames \;$\rightarrow$\; generative model \;$\rightarrow$\; generated segment}\\[2pt]
\raisebox{-0.5\height}{\includegraphics[height=1.7cm]{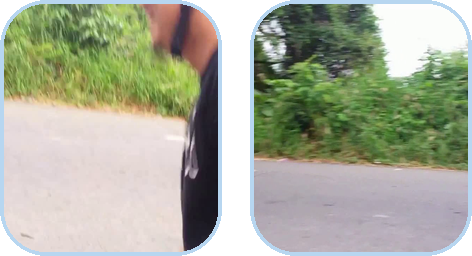}}\qquad
\raisebox{-0.5\height}{\includegraphics[height=1.7cm]{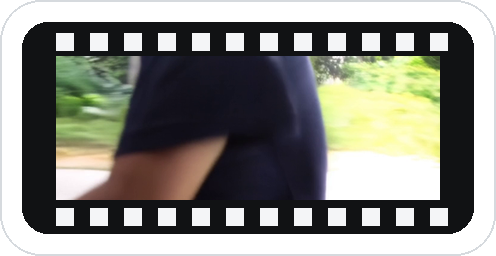}}\\[6pt]
{\footnotesize Constructed video after re-inserting the generated segment (boxed in red)}\\[2pt]
\includegraphics[width=0.92\linewidth]{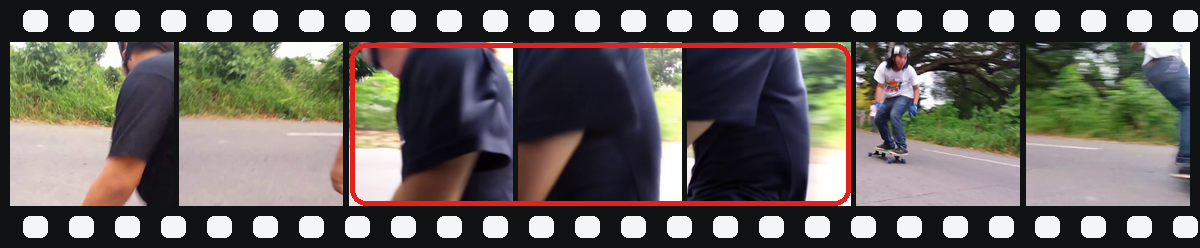}
\caption{Data construction with \textbf{SkyReels-V2-DF-14B}. Source video: ActivityNet (\emph{Longboarding}); generation interval $[1.20, 3.40]$\,s of a $5$\,s clip.}
\label{fig:dg-skyreels-longboarding}
\end{figure*}

\begin{figure*}[tbp]\centering
{\footnotesize Original real video}\\[2pt]
\includegraphics[width=0.92\linewidth]{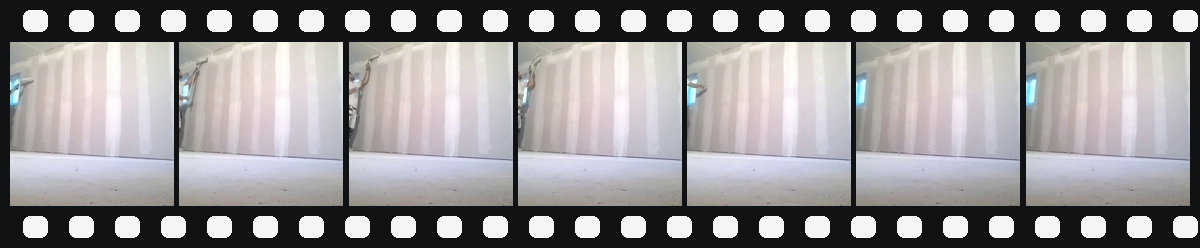}\\[6pt]
{\footnotesize First/last frames \;$\rightarrow$\; generative model \;$\rightarrow$\; generated segment}\\[2pt]
\raisebox{-0.5\height}{\includegraphics[height=1.7cm]{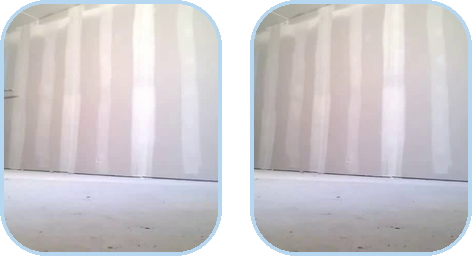}}\qquad
\raisebox{-0.5\height}{\includegraphics[height=1.7cm]{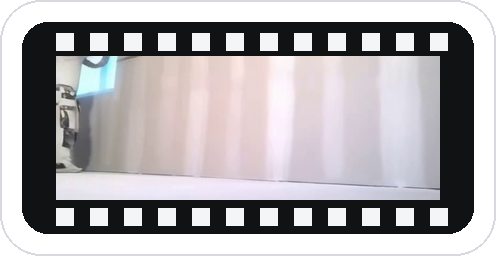}}\\[6pt]
{\footnotesize Constructed video after re-inserting the generated segment (boxed in red)}\\[2pt]
\includegraphics[width=0.92\linewidth]{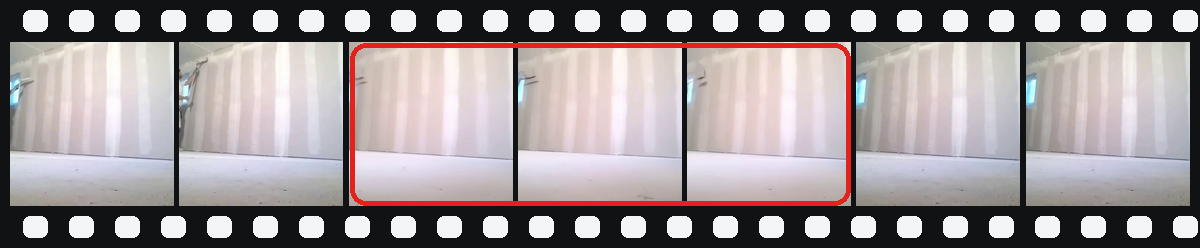}
\caption{Data construction with \textbf{SkyReels-V2-DF-14B}. Source video: ActivityNet (\emph{Plastering}); generation interval $[1.20, 3.40]$\,s of a $5$\,s clip.}
\label{fig:dg-skyreels-plastering}
\end{figure*}

\end{document}